\documentclass[letterpaper]{article} 
\usepackage{aaai23}  
\usepackage{times}  
\usepackage{helvet}  
\usepackage{courier}  
\usepackage[hyphens]{url}  
\usepackage{graphicx} 
\usepackage{natbib}  
\usepackage{caption} 
\usepackage{algorithm}
\usepackage{algorithmic}
\usepackage{mathtools}
\usepackage{mymacros}
\usepackage{mathrsfs}
\usepackage{bbm}
\usepackage{amsmath}
\usepackage{microtype} 

\usepackage{newfloat}
\usepackage{listings}
\DeclareCaptionStyle{ruled}{labelfont=normalfont,labelsep=colon,strut=off} 
\floatstyle{ruled}
\newfloat{listing}{tb}{lst}{}
\floatname{listing}{Listing}
\allowdisplaybreaks[4]
\title{Simultaneously Updating All Persistence Values in Reinforcement Learning}
\author{
    Luca Sabbioni\textsuperscript{\rm 1}, Luca Al Daire\textsuperscript{\rm 1}, Lorenzo Bisi\textsuperscript{\rm 2}, Alberto Maria Metelli\textsuperscript{\rm 1}, Marcello Restelli\textsuperscript{\rm 1}
}
\affiliations{
    \textsuperscript{\rm 1}Politecnico di Milano, Milan, Italy\\
     \textsuperscript{\rm 2}ML cube, Milan, Italy\\

    luca.sabbioni@polimi.it
}

\usepackage{bibentry}
\allowdisplaybreaks[4]
\begin{document}
 \setlength{\parskip}{1.3mm plus2mm minus2mm}
 
\maketitle

\begin{abstract}
In \emph{reinforcement learning}, the performance of learning agents is highly sensitive to the choice of time discretization. Agents acting at high frequencies have the best control opportunities, along with some drawbacks, such as possible inefficient exploration and vanishing of the action advantages. The repetition of the actions, i.e., \textit{action persistence}, comes into help, as it allows the agent to visit wider regions of the state space and improve the estimation of the action effects.
In this work, we derive a novel \textit{All-Persistence Bellman Operator}, which allows an effective use of both the low-persistence experience, by decomposition into sub-transition, and the high-persistence experience, thanks to the introduction of a suitable \emph{bootstrap} procedure. 
In this way, we employ transitions collected at \emph{any} time scale to update simultaneously the action values of the considered persistence set.
We prove the contraction property of the All-Persistence Bellman Operator and, based on it, we extend classic Q-learning and DQN. After providing a study on the effects of persistence, we experimentally evaluate our approach in both tabular contexts and more challenging frameworks, including some Atari games.
\end{abstract}

\section{Introduction}
In recent years, Reinforcement Learning~\citep[RL,][]{sutton2018reinforcement} methods have proven to be successful in a wide variety of applications, including robotics~\citep{kober2013learning,gu2017deep,haarnoja2019learning, kilinc2019reinforcement}, autonomous driving~\citep{kiran2021deep} and continuous control~\citep{lillicrap2015continuous, schulman2017proximal}. These sequential decision-making problems are typically modelled as a Markov Decision Process~\citep[MDP,][]{puterman2014markov}, a formalism that addresses the agent-environment interactions through \emph{discrete-time} transitions.  \emph{Continuous-time} control problems, instead, are usually addressed by means of time discretization, which induces a specific control frequency $f$, or, equivalently, a time step $\delta=\frac{1}{f}$~\citep{park2021time}. This represents an environment hyperparameter, which may have dramatic effects on the process of learning the optimal policy~\citep{metelli2020control,kalyanakrishnan2021analysis}. Indeed, higher frequencies allow for greater control opportunities, but they have significant drawbacks. The most relevant one is related to the toned down effect of the selected actions. In the limit for time discretization  $\delta\rightarrow 0$, the advantage of each action collapses to zero, preventing the agent from finding the best action~\citep{baird1994reinforcement, tallec2019time}. Even policy gradient methods are shown to fail in this (near-)continuous-time setting, and the reason is related to the divergent variance of the gradient~\citep{park2021time}. 
The consequences of each action might be detected if the dynamics of the environment has the time to evolve, hence with an agent acting with higher frequencies lead to higher sample complexity. 

Another consequence of the use of high frequencies is related to the difficulty of exploration. A random uniform policy played at high frequency may not be adequate, as in some classes of environments, including the majority of real-world control problems, it tends to visit only a local neighborhood of the initial state~\citep{amin2020locally, park2021time, yu2021taac}. This is problematic, especially in goal-based or sparse rewards environments, where the most informative states may never be visited. On the other hand, large time discretizations benefit from a higher probability of reaching far states, but they also deeply modify the transition process, hence a possibly large subspace of states may not be reachable. 

One of the solutions to achieve the advantages related to exploration and sample complexity, while keeping the control opportunity loss bounded, consists in \textit{action persistence} or \textit{action repetition} \cite{schoknecht2003reinforcement, braylan2015frame, lakshminarayanan2017dynamic, metelli2020control}, which is equivalent to acting at lower frequencies. 
Thus, the agent can achieve, in some environments, a more effective exploration,
better capture the consequences of each action, and fasten convergence to the optimal policy.

In this work, we propose a value-based approach in which the agent does not only choose the \emph{action}, but also its \emph{persistence}, with the goal of making the most effective use of samples collected at different persistences.
The main contribution of this paper is a general approach in which information collected from the interaction with the environment at \emph{one} persistence is used to improve the action value function estimates of \emph{all} the considered possible persistences. On one hand, the $\overline{\kappa}$-step transitions can be decomposed in many sub-transitions of reduced length and used to update lower persistence $k \le \overline{\kappa}$ value functions.  On the other hand, they represent partial information for the estimation of the effects of higher persistence $k> \overline{\kappa}$ actions. Indeed, they can be employed to update the estimates  by using a suitable \textit{bootstrapping} procedure of the missing information.
This means that, after the interaction with the environment, according to the action-persistence pair selected by the agent, all value function estimates are updated simultaneously for each of the available persistences $k \in \mathcal{K}$. We formalize, in Section \ref{sec:update}, this procedure by introducing the \emph{All-persistence Bellman Operator}. We prove that such an  operator enjoys a contraction property analogous to that of the traditional optimal Bellman operator.
Consequently, in Section \ref{sec:perq} we embed the All-persistence Bellman operator into the classic Q-learning algorithm, obtaining \emph{\algnameExtQL} (\algnameQL). This novel algorithm, through an effective use of the transitions sampled at different persistences, displays two main advantages. First, since each individual transition is employed to update the value function estimates at different persistences, we experience a faster convergence. Second, the execution of persistent actions, given the nature of a large class of environments, fosters exploration of the state space, with a direct effect on the learning speed (Section \ref{sec:advantages}).
Furthermore, to deal with more complex domains, we move, in Section \ref{sec:perdq}, to the Deep RL scenario, extending the Deep Q-Network (DQN) algorithm to its persistent version, called \emph{\algnameExtDQN} (\algnameDQN). Finally, in Section \ref{sec:exps}, we evaluate the proposed algorithms, in comparison with state-of-the-art approaches, on illustrative and complex domains, highlighting strengths and weaknesses.

\section{Related Work}
The first work extending RL agents with action repetition, go back to \citealt{schoknecht2003reinforcement}. In this paper,  multi-step actions (MSAs) were introduced, reducing the number of decisions needed to reach the goal and making the time scale coarser. 
Action persistence has acquired practical relevance since the introduction of Deep RL~\cite{mnih2013playing}, by leveraging the
 \textit{frame skip} parameter \cite{bellemare2013arcade}.
Several works \cite{braylan2015frame, khan2019optimal,metelli2020control} had shown the importance of persistence for helping exploration and policy learning. Among these works, \citealt{dabney2020temporally} introduced an $\epsilon z-$greedy exploration, with a random exploratory variable deciding the duration of each action.
As explained in \citealt{metelli2020control}, changing frequency deeply modifies the underlying MDP, as a special instance of a configurable MDP \cite{metelli2018configurable}, where environmental parameters can be tuned to improve the performance. Indeed, in \citealt{grigsby2021towards} the authors proposed an algorithm to automatically tune the control frequency, along with other learning hyperparameters. \citealt{mann2015approximate} illustrates that approximate value iteration techniques can converge faster with action persistence (seen as \emph{options} with longer duration).

Action repetition has many advantages, but it may reduce the control opportunities. Consequently, researchers have been trying to include the possibility to \emph{dynamically} change the control frequency during learning: in Augmented-DQN  \cite{lakshminarayanan2017dynamic}, the action space is duplicated to be able to choose actions with two previously selected repetition rates. 

A different approach is proposed in \citealt{sharma2017learning}, where a \textit{skip network} is used to the action persistence in a specific state, regardless of the chosen action.
One way to differentiate persistences with actions is proposed by TempoRL \cite{Biedenkapp2021TempoRL}, where the skip network depends on both state and action (and the estimated Q-value function) to evaluate the effects for different possible frequencies. In \citealt{bellemare2016increasing},  \emph{persistent advantage learning} is proposed in which the advantage learning is overridden by a Q-learning update with repeated actions, when the latter promises better Q-values.

In the framework of policy-gradient methods, persistence is introduced in \citealt{yu2021taac}, with the introduction of a secondary policy for choosing whether to repeat the previous action or to change it according to the principal agent. 
 A completely different approach is presented by \citealt{park2021time}: the authors claim that when $\delta \to 0$ policy-based methods tend to degrade. Thanks to the introduction of a \emph{safe region}, the agent keeps repeating an action until the distance of the visited states overcomes a certain threshold. This state locality can guarantee reactivity, especially in some environments where the blind repetition of an action can be dangerous.
\section{Preliminaries}
In this section, we provide the necessary background employed in the following of the paper.

\paragraph{Mathematical Background}
Given a measurable space $(\mathcal{X},\sigma_{\mathcal{X}})$, where $\mathcal{X}$ is a set and  $\sigma_{\mathcal{X}}$ a $\sigma$-algebra, we denote with $\mathscr{P}(\mathcal{X})$ the set of probability measures and with $\mathscr{B}(\mathcal{X})$ the set of the bounded measurable functions. We denote with $\delta_{x}$ the Dirac measure centered in $x \in \mathcal{X}$. Let $f \in \mathscr{B}(\mathcal{X})$, the $L_{\infty}$-norm is defined as $\norm[\infty]{f} = \sup_{x \in \mathcal{X}} f(x)$.\newline\vspace*{-2mm}

\paragraph{Markov Decision Processes}
A discrete-time Markov Decision Process~\citep[MDP,][]{puterman2014markov} is defined as a tuple $\mathcal{M}\coloneqq\langle\Sspace,\Aspace, P, r, \gamma\rangle$, where $\Sspace$ is the  state space, $\Aspace$ the finite action space, $P: \Ss \times \As \rightarrow \mathscr{P}(\Ss)$ is the Markovian transition kernel,  
$r:\Ss \times \As \rightarrow \mathbb{R}$ is the reward function, bounded as $\|r\|_{\infty}\le \Rmax<+\infty$, and $\gamma\in[0,1)$ is the discount factor. 
A Markovian stationary policy $\pi: \Sspace\rightarrow\mathscr{P}(\Aspace)$ maps states to probability measures over $\mathcal{A}$. We denote with $\Pi$ the set of  Markovian stationary policies. The \emph{action-value function}, or $Q$-function, of a policy $\pi \in \MSPol$ is the expected discounted sum of the rewards obtained by
performing action $a$ in state $s$ and following policy $\pi$ thereafter:
$Q^{\pi}(s,a) = \EV_{\pi} \big[ \sum_{t=0}^{+\infty} \gamma^t r_{t+1} \rvert s_0=s,\, a_0=a\big]$,
where $r_{t+1} = r(s_t,a_t)$, $a_{t} \sim \pi(\cdot|s_{t})$, and $s_{t+1} \sim P(\cdot|s_{t},a_t)$ for all $t \in \Naturals$.
The optimal $Q$-function is given by: $Q^\star(s,a) = \sup_{\pi \in \Pi} Q^\pi(s,a)$ for all $(s,a) \in \SAs$. 
A policy $\pi$ is \emph{greedy} w.r.t. a function $f \in \mathscr{B}(\SAs)$ if it plays only greedy actions, i.e., $\pi(\cdot|s) \in \mathscr{P} \left( \argmax_{a \in \Aspace} f(s,a) \right)$.
An \emph{optimal policy} $\pi^\star \in \Pi$ is any policy greedy w.r.t. $Q^\star$.

\paragraph{$Q$-learning}
The \emph{Bellman Optimal Operator} 
$T^{\star}:\mathscr{B}(\SAs)\rightarrow\mathscr{B}(\SAs)$ is defined for every $f\in\mathscr{B}(\SAs)$ and $(s,a) \in \SAs$ as in~\citep{bertsekas2004stochastic}: $(T^{\star} f)(s,a) = r(s,a) + \gamma \int_{\Ss} P(\de s'|s,a) \max_{a' \in \As} f(s',a').$
$T^{\star}$ is a $\gamma$-contraction in $L_{\infty}$-norm and its unique fixed point is the optimal $Q$-function.
When $P$ and $r$ are known, value-iteration \cite{puterman2014markov} allows computing $Q^\star$ via iterative application of $T^{\star}$. When the environment is unknown, $Q$-learning \cite{watkins1989learning} collects samples with a \textit{behavioral} policy (e.g., $\epsilon$-greedy) and then updates Q-function estimate based on the updated rule: $Q(s_t,a_t) \leftarrow (1-\alpha) Q(s_t,a_t) + \alpha (r_{t+1} + \gamma \max_{a' \in \Aspace} Q(s_{t+1}, a'))$. where $\alpha > 0$ is the learning rate. 

\paragraph{Deep $Q$-Networks}
Deep $Q$-Network (DQN, \citealt{mnih2013playing, mnih2015human}) employs a deep neural network  with weights $\vtheta$ to learn an approximation $Q_{\vtheta}$ of $Q^\star$. The transitions are stored in the \emph{replay buffer} $\mathcal{D}=\{(s_t,a_t,r_{t+1}, s_{t+1})\}_{t=1}^n$ to mitigate temporal correlations. To improve stability, a \emph{target network}, whose parameters $\vtheta^{-}$ are kept fixed for a certain number of steps, is employed. The $Q$-Network is trained to minimize the mean squared temporal difference error $r + \gamma \max_{a' \in \mathcal{A}} Q_{{\vtheta}^{-}}(s',a')- Q_{\vtheta}(s,a)$ on a batch of tuples sampled from the replay buffer $(s,a,r,s') \sim \mathcal{D}$.

\paragraph{Action Persistence}
The execution of actions with a persistence $k \in \Naturals$ can be modeled by means of the $k$-persistent MDP~\cite{metelli2020control}, characterized by the $k$-persistent transition model $P_k$ and reward function $r_k$. To formally define them, the \emph{persistent transition model} is introduced: $P^{\delta} (\cdot,\cdot|s,a) = \int_{\Ss} P(\de s'| s,a) \delta_{(s',a)}(\cdot,\cdot)$, which replicates in the next state $s'$ the previous action $a$. Thus, we have $P_k(\cdot|s,a) = \left((P^{\delta})^{k-1}P\right)(\cdot|s,a)$ and $r_k(s,a) = \sum_{i=0}^{k-1} \gamma^i \left((P^{\delta})^{i} r\right)(s,a)$.
This framework eases the analysis of \textit{fixed} persistences, but it does not allow action repetition for a \textit{variable} number of steps.
\begin{algorithm}[t]
\medmuskip=0mu
\thinmuskip=0mu
\thickmuskip=0mu
\caption{All Persistence Bellman Update} \label{alg:pers_update}
\small
\begin{algorithmic}[1]
\REQUIRE Sampling persistence $\ksamp_t$, partial history $H_{t}^{\ksamp_t}$,\\ Q-function $Q$.
\ENSURE Updated $Q$-function $Q$
\FOR{$j=\ksamp_t, \ksamp_t-1 \dots, 1$}
\FOR{$i=j-1, j-2,\dots, 0$}
\STATE{$k \leftarrow j-i$}
\STATE{\textcolor{vibrantMagenta}{$Q(s_{t+i}, a_t, k) \leftarrow (1-\alpha)Q(s_{t+i}, a_t, k) +$}}
\STATE{\textcolor{vibrantMagenta}{$\quad \quad \quad \quad \quad \quad \quad \quad \alpha\widehat{T}^{\star}_{t+i} Q(s_{t+i}, a_t, k)$}}
\FOR{$d=1, 2,\dots, K_{\max}-k$}
\STATE{\textcolor{vibrantBlue}{$Q(s_{t+i}, a_t, k+d) \leftarrow (1-\alpha)Q(s_{t+i}, a_t, k+d) +$}} 
\STATE{\textcolor{vibrantBlue}{$\quad \quad \quad \quad \quad \quad \quad \quad \alpha{\widehat{T}}^{k}_{t+i} Q(s_{t+i}, a_t, k+d)$}}
\ENDFOR
\ENDFOR
\ENDFOR
\end{algorithmic}
\end{algorithm}

\begin{algorithm}[t]
\caption{\algnameExtQL (\algnameQL)} \label{alg:PerQ}
\begin{algorithmic}[1]
\small
\REQUIRE Learning rate $\alpha$, exploration coefficient $\epsilon$,\\ number of episodes $N$
\ENSURE $Q$-function
\STATE{Initialize $Q$ arbitrarily, 
$Q(terminal, \cdot, \cdot) = 0$}
\FOR{$episode=1, \dots, N$}
\STATE{$t \leftarrow 0$}
\WHILE{$s_t$ is not $terminal$}
\STATE{$a_t, \ksamp_t \ \sim \ \psi^{\epsilon}_Q(s_t)$}
\FOR{$\tau=1,\dots,\ksamp_t$}
\STATE{Take action $a_t$, observe $s_{t+\tau}, r_{t+\tau}$}
\ENDFOR
\STATE{Store partial history $H_{t}^{\ksamp_{t}}$}
\STATE{Update $Q$ according to Alg.\ref{alg:pers_update}}
\STATE{$t \leftarrow t+\ksamp_t$}
\ENDWHILE
\ENDFOR
\end{algorithmic}
\end{algorithm}

\section{All-Persistence Bellman Update}\label{sec:update}
In this section, we introduce our approach to make effective use of the samples collected at \emph{any} persistence. We first introduce the notion of \emph{persistence option} and then, we present the \emph{all-persistence Bellman operator}.

\subsection{Persistence Options}\label{sec:persistenceOptions}\vspace*{-.1cm}
We formalize the decision process in which the agent chooses a primitive action $a$ together with its persistence $k$. To this purpose, we introduce the \emph{persistence option}.

\begin{defi}\label{defi:persistenceOption}
Let  $\Aspace$ be the space of \emph{primitive} actions of an MDP $\mathcal{M}$ and $\Kspace \coloneqq \{1, \dots, K_{\max}\}$, where $K_{\max} \ge 1$, be the set of persistences. A \emph{persistence option} $o \coloneqq (a,k) $ is the decision of playing primitive action $a \in \Aspace$ with persistence $k \in \Kspace$. We denote with $\mathcal{O}^{(k)} \coloneqq \{(a,k) \,:\, a \in \Aspace\}$ the set of options with fixed persistence $k \in \Kspace$ and $\mathcal{O} \coloneqq \bigcup_{k\in \mathcal{K}}\Ospace^{(k)} = \Aspace \times \Kspace$.
\end{defi}
The decision process works as follows. At time $t=0$, the agent observes $s_0 \in \Sspace$, selects a persistence option $o_0 = (a_0,k_0) \in \mathcal{O}$, observes the sequence of states $(s_1,\dots, s_{k_0})$ generated by repeating primitive action $a_0$ for $k_0$ times, \ie $s_{i+1} \sim P(\cdot|s_{i},a_0)$ for $i \in \{0,\dots,k_0-1\}$, and the sequence of rewards $(r_1,\dots,r_{k_0})$ with $r_{i+1} = r(s_i,a_0)$ for $i \in \{0,\dots,k_0-1\}$.
Then, in state $s_{k_0}$ the agent selects another option $o_1 = (a_1,k_1) \in \mathcal{O}$ and the process repeats.
During the execution of the persistence option, the agent is not allowed to change the primitive action.\footnote{From this definition, it follows that $\Aspace$ is isomorphic to $\mathcal{O}^{(1)}$.}

\begin{remark} \textup{(\textbf{Persistence and Options})}~
The persistence option (Definition~\ref{defi:persistenceOption}) is in all regards a \emph{semi-Markov option}~\cite{precup2001temporal}, where the initiation set is the set of all states $\Ss$, the termination condition depends on time only, and the intra-option policy is constant. Indeed, the described process generates a \emph{semi-Markov decision process} \citep{puterman2014markov}, fully determined by the behavior of $\mathcal{M}$, as shown in \cite{sutton1999between}.
\end{remark}
\begin{figure}[t]
\centering
\includegraphics[width=.8\columnwidth]{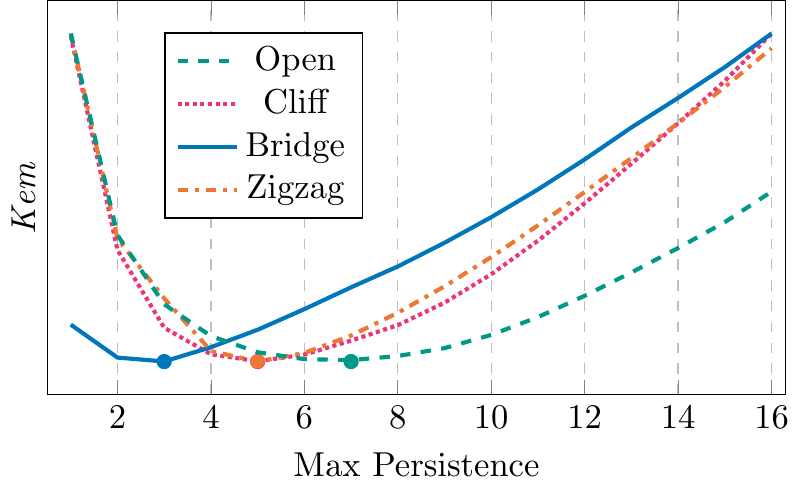}
\caption{Normalized Kemeny's constant in tabular environments as function of $K_{\max}$. Bullets represent the minimum.}\label{fig:Kemeny}
\end{figure}

\begin{remark} \textup{(\textbf{Persistence Options vs Augmented Action Space})}~
There is an important difference between using persistence options $\mathcal{O}$ in the original MDP $\mathcal{M}$ and defining an augmented MDP $\mathcal{M}_{\Kspace}$ with new action space $\Aspace \times \Kspace$ and properly redefined transition model and reward function~\cite{lakshminarayanan2017dynamic}: when executing a persistence option $o_t=(a_t,k_t) \in \mathcal{O}$ at time $t$, we observe the \emph{full} sequence of states $(s_{t+1}, \dots,s_{t+k_t} )$ and rewards $(r_{t+1},\dots,r_{t+k_t})$. Instead, in the augmented MDP $\mathcal{M}_{\Kspace}$ we \emph{only} observe the last state $s_{k_t}$ and the cumulative reward $r_{t+1}^k = \sum_{i=0}^{k-1} \gamma^i r_{t+i+1}$.  We will heavily exploit the particular option structure, re-using fragments of experience to perform \textit{intra-option learning}.
\end{remark}

We now extend the policy and state-action value function definitions to consider this particular form of options. A \emph{Markovian stationary policy over persistence options} $\psi : \Ss \rightarrow \mathscr{P}(\Ospace)$ is a mapping between states and probability measures over persistence options. We denote with $\Psi$ the set of the policies of this nature.
The state-option value function $Q^{\psi}: \Sspace \times \mathcal{O} \rightarrow \mathbb{R}$ following a policy over options $\psi \in \Psi$  is defined as $Q^{\psi}(s,a,k) \coloneqq \EV_{\psi} \big[ \sum_{t=0}^{+\infty} \gamma^t r_{t+1} \rvert s_0=s,\, a_0=a, k_0=k\big]$.
In this context, the optimal action-value function is defined as: $\QstarK(s,a,k)=\sup_{\psi\in\Psi}Q^\psi(s,a,k) $.

\subsection{All-Persistence Bellman Operator}\label{sec:allPersistenceOp}
Our goal is to leverage any $\ksamp$-persistence transition to learn $\QstarK(\cdot,\cdot,k)$ for \textit{all} the possible action-persistences in $k \in \Kspace$. Suppose that $\ksamp \ge k$, then, we can exploit any sub-transition of $k$ steps from the $\ksamp$-persistence transition to update the value $\QstarK(\cdot,\cdot,k)$.
Thus, we extend the Bellman optimal operator to persistence options $\TstarK: \mathscr{B}(\Sspace \times \mathcal{O}) \rightarrow \mathscr{B}(\Sspace \times \mathcal{O})$ with $f \in \mathscr{B}(\Sspace \times \mathcal{O})$:
{
\medmuskip=0mu
\thinmuskip=0mu
\thickmuskip=0mu
\begin{equation*}
    \left(\TstarK f\right) (s,a,k)=r_{k}(s,a) + \gamma^{k}\int_{\Sspace} P_{k}(\mathrm{d}s'|s,a)\hspace{0.2mm}\max_{(a',k')\in\mathcal{O}}f(s',a',k').
\end{equation*}
}

If, instead, $\ksamp < k$, in order to update the value $\QstarK(\cdot,\cdot,k)$, we partially exploit the $\ksamp$-persistent transition, but then, we need to \textit{bootstrap} from a lower persistence $Q$-value, to compensate the remaining $k-\ksamp$ steps. To this end, we introduce the \emph{bootstrapping} operator \newline $T^{\ksamp}: \mathscr{B}(\Sspace \times \mathcal{O}) \rightarrow \mathscr{B}(\Sspace \times \mathcal{O})$ with $f \in \mathscr{B}(\Sspace \times \mathcal{O})$:
{
\medmuskip=0mu
\thinmuskip=0mu
\thickmuskip=0mu
\begin{equation*}
    \left(T^{\ksamp}f\right) (s,a,k) = r_{\ksamp}(s,a) + \gamma^{\ksamp}\int_{\Sspace} P_{\ksamp}(\mathrm{d}s'|s,a) f(s',a,k-\ksamp).
\end{equation*}
}

By combining these two operators, we obtain  the \textit{All-Persistence Bellman operator} $\mathcal{H_{\ksamp}}: \mathscr{B}(\Sspace \times \mathcal{O}) \rightarrow \mathscr{B}(\Sspace \times \mathcal{O})$ defined for every $f \in \mathscr{B}(\Sspace \times \mathcal{O})$ as: $(\Hop^{\ksamp} f) (s,a,k) = \left((\mathbbm{1}_{k\le\ksamp}\TstarK+
    \mathbbm{1}_{k>\ksamp}T^{\ksamp})f\right)(s,a,k)$.
Thus, given a persistence $\ksamp \in \Kspace$, $\Hop^{\ksamp}$ allows updating all the $Q$-values with $k \leq \ksamp$ by means of $\TstarK$, and all the ones with $k > \ksamp$ by means of $T^{\ksamp}$. The following result demonstrates its soundness.

\begin{restatable}[
]{thm}{HQcontraction}\label{prop:HQ_contraction}
The all-persistence Bellman operator $\mathcal{H}^{\ksamp}$ fulfills the following properties:
\begin{enumerate}
    \item $\mathcal{H}^{\ksamp}$ is a $\gamma$-contraction in $L_\infty$ norm;
    \item $\QstarK$ is its unique fixed point;
    \item $\QstarK$ is monotonic in $k$, \ie  for all $(s,a) \in \SAs$ if $k \le k'$ then $\QstarK(s,a,k) \ge \QstarK(s,a,k')$.
\end{enumerate}
\end{restatable}

 Thus, operator $\mathcal{H}^{\ksamp}$ contracts to the optimal action-value function $\QstarK$, which, thanks to monotonicity, has its highest value at the lowest possible persistence. In particular, it is simple to show that $\QstarK(s,a,1) = Q^\star(s,a)$ for all $(s,a) \in \SAs$ (Corollary~\ref{cor:apx_Qstar_equiv}), \ie by fixing the persistence to $k=1$ we retrieve the optimal $Q$-function in the original MDP, and consequently, we can reconstruct a greedy optimal policy. This highlights that the primitive action space leads to the same optimal Q-function as with persistence options. Persistence, nevertheless, is helpful for exploration and learning, but for an optimal persistent policy $\psi^*$, there exists a primitive policy $\pi^*$ with the same performance.

\begin{figure}[t]
\centering
\includegraphics[width=\columnwidth]{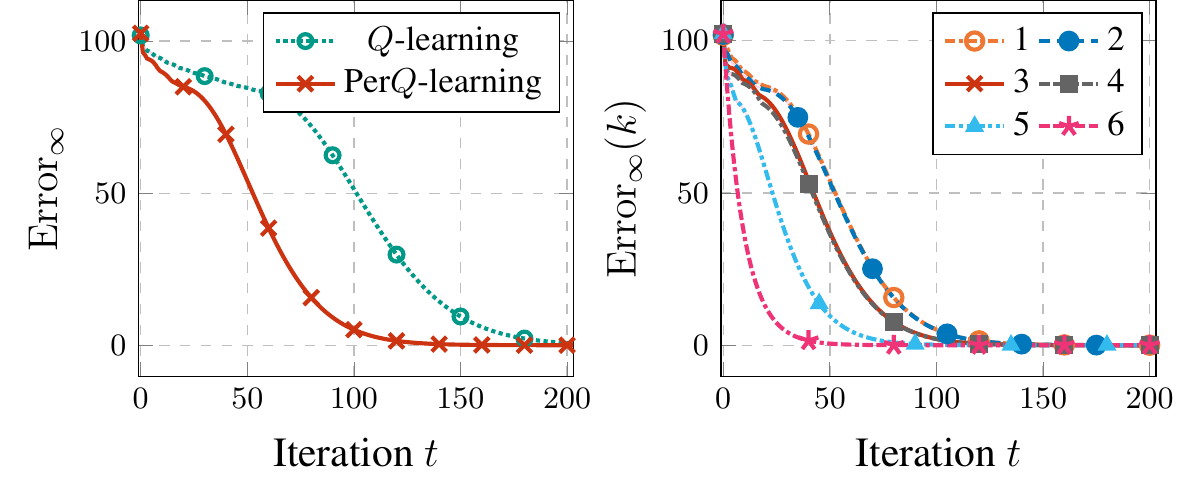}
\caption{$L_{\infty}$ error on 6x6 grid-world between synchronous $Q$-learning and \algnameQL (left) and for different persistence options $k \in\{1,...,6\}$ of \algnameQL (right). (100 runs, avg $\pm$ 95 \% c.i.)}\label{fig:sync}
\vspace*{-2mm}
\end{figure}

\section{Persistent $Q$-learning}\label{sec:perq}
It may not be immediately clear what are the advantages of $\mathcal{H}^{\ksamp}$ over traditional updates. These become apparent with its empirical counterpart $\widehat{\mathcal{H}}_t^{\ksamp} = \mathbbm{1}_{k\le\ksamp}\widehat{T}_t^\star+\mathbbm{1}_{k>\ksamp}\widehat{T}^{\ksamp}_t$, where:
\begin{align*}
\left(\widehat{T}^{\star}_{t} Q\right) (s_t,a_t,k)&=
r_{t+1}^{k}
+ \gamma^{k}\hspace{0.2mm}\max_{(a',k')\in\mathcal{O}}Q(s_{t+k},a',k'),\\
\left(\widehat{T}^{\ksamp}_{t} Q\right) (s_t,a_t,k) &= 
r_{t+1}^{\ksamp}
+ \gamma^{\ksamp}Q(s_{t+k},a_t,k-\ksamp).
\end{align*}
These empirical operators depend on the current \textit{partial history}, which we define as: $H_t^{\ksamp}\coloneqq (s_t, a_t, r_{t+1}, s_{t+1}, r_{t+2},\dots, s_{t+\ksamp})$, used by Algorithm \ref{alg:pers_update} to update each persistence in a backward fashion, as illustrated also in Appendix \ref{sec:apx scheme}. At timestep $t$, given a sampling persistence $\ksamp_t$, for all sub-transitions of $H_t^{\ksamp}$, starting at $t+i$ and ending in $t+j$, we apply $\widehat{\mathcal{H}}_t^{j-i}$ to $Q(s_{t+i}, a_t, k+d)$, for all $d \leq K_{\max}-k$, where $k=j-i$.

With these tools, it is possible to extend $Q$-learning \cite{watkins1989learning} to obtain the Persistent $Q$-learning algorithm (abbreviated as Per$Q$-learning), described in Algorithm \ref{alg:PerQ}. The agent follows a policy $\psi^{\epsilon}_Q$, which is $\epsilon$-greedy w.r.t. the option space and the current $Q$-function.

This approach extends the MSA-$Q$-learning \cite{schoknecht2003reinforcement}, by bootstrapping higher persistence action values from lower ones. More precisely, both methods apply the update related to  $\widehat{T}^\star$, but MSA-$Q$-learning does not use $\widehat{T}^{\ksamp}$ instead. As shown in the empirical analysis, in some domains this difference can be crucial to speed up the convergence. Similarly to MSA-$Q$-learning, we perform backwards updates to allow for an even faster propagation of values. The proposed approach also differs from TempoRL $Q$-learning \cite{Biedenkapp2021TempoRL}, where action-persistence is selected using a dedicated value-function, learned separately from the $Q$-function.
The asymptotic convergence of \algnameExtQL to $\QstarK$ directly follows \cite{singh2000convergence}, being $\mathcal{H}^{\ksamp}$ a contraction and since their (mild) assumptions are satisfied.

\begin{algorithm}[t]
\caption{Multiple Replay Buffer Storing} \label{alg:split_traj}
\begin{algorithmic}[1]
\small
\REQUIRE Maximum persistence $K_{\max}$, replay buffers $(\mathcal{D}_k)_{k=1}^{K_{\max}}$, transition tuple $(s_t,a_t, \overline{\kappa}_t, H_t^{\overline{\kappa}_t})$.
\vspace{0.1cm}
\FOR{$k=1, \dots, K_{\max}$}
\FOR{$\tau=0,  \dots, \max\{\overline{\kappa}_t-k,0\}$}
    \STATE{$\mathcal{D}_k \leftarrow \mathcal{D}_k \cup (s_{t+\tau},a_t, s_{t+\tau+k},  r_{t+1+\tau}^k, k)$}
\ENDFOR
\FOR{$\tau=1,  \dots, \min\{\overline{\kappa}_t,k-1\}$}
    \STATE{$\mathcal{D}_k \leftarrow \mathcal{D}_k \cup (s_{t+\overline{\kappa}_t-\tau},a_t, s_{t+\overline{\kappa}_t}, r_{t+1+\overline{\kappa}_t-\tau}^{\tau}, \tau)$}
\ENDFOR
\ENDFOR
\end{algorithmic}
\end{algorithm}

\section{Empirical Advantages of Persistence}\label{sec:advantages}
In this section, we provide some numerical simulations to highlight the benefits of our approach. The settings are illustrative, to ease the detection of the individual advantages of persistence, before presenting more complex applications.

\paragraph{Exploration}
One of the main advantages of persistence is related to faster exploration, especially in goal-based environments (e.g., robotics and locomotion tasks). Indeed, persisting an action allows reaching faster states far from the starting point and, consequently, propagating faster the reward. The reason is due to the increased chances of $1$-persistent policies to get stuck in specific regions. As explained in \cite{amin2020locally}, persistence helps to achieve \textit{self-avoiding} trajectories, by increasing the expected return time in previously visited states. Hence, we study the effects of a persisted exploratory policy on the MDP, \ie a policy $\psi \in \Psi$ over persistence options $\mathcal{O}$ (details in Appendix \ref{sec:apx kemeny}). 

To this purpose, we compute the \textit{Kemeny's constant} \cite{catral2010kemeny,patel2015robotic}, which corresponds to the expected first passage time from an arbitrary starting state $s$ to another one $s'$ under the stationary distribution induced by $\psi$. We consider four discrete tabular environments: \textit{Open} is a 10x10 grid with no obstacles, while the others, presented in \cite{Biedenkapp2021TempoRL}, are depicted in Figure \ref{fig:tabular}. In Figure \ref{fig:Kemeny}, we plot the variations of Kemeny's constant as a function of the maximum persistence $K_{\max}$, while following a uniform policy $\psi$ over $\mathcal{O}$. We observe that increasing $K_{\max}$ promotes exploration, and highlights the different $K_{\max}$ attaining the minimum value of the constant, due to the different complexity of the environments.

\begin{figure}[t]
\includegraphics[width=\columnwidth]{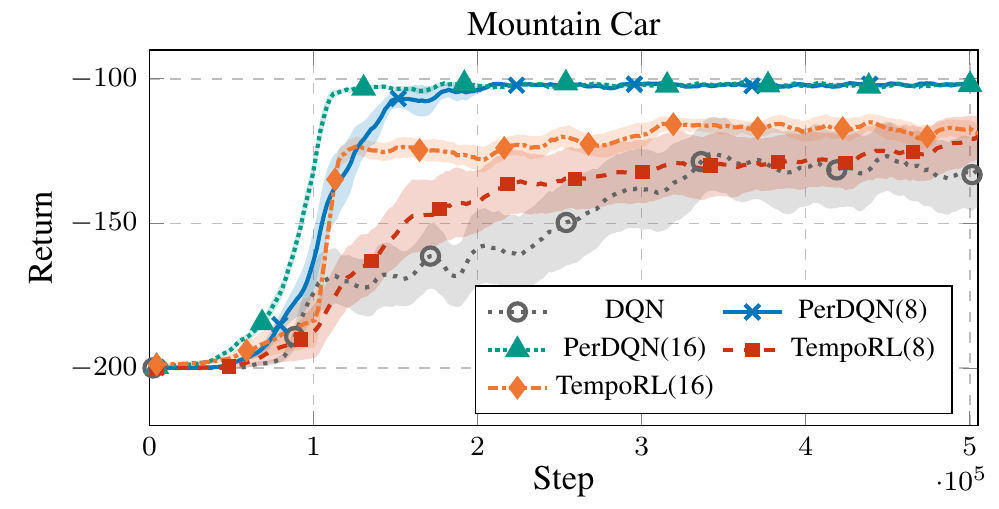}
\vspace*{-2mm}\caption{MountainCar results. Parenthesis in the legend denote $K_{\max}$. 20 runs (avg$\pm$ 95\% c.i.). }\label{fig:mtcar}\vspace*{-2mm}
\end{figure}
\begin{figure*}
\centering
\includegraphics[width=.9\textwidth]{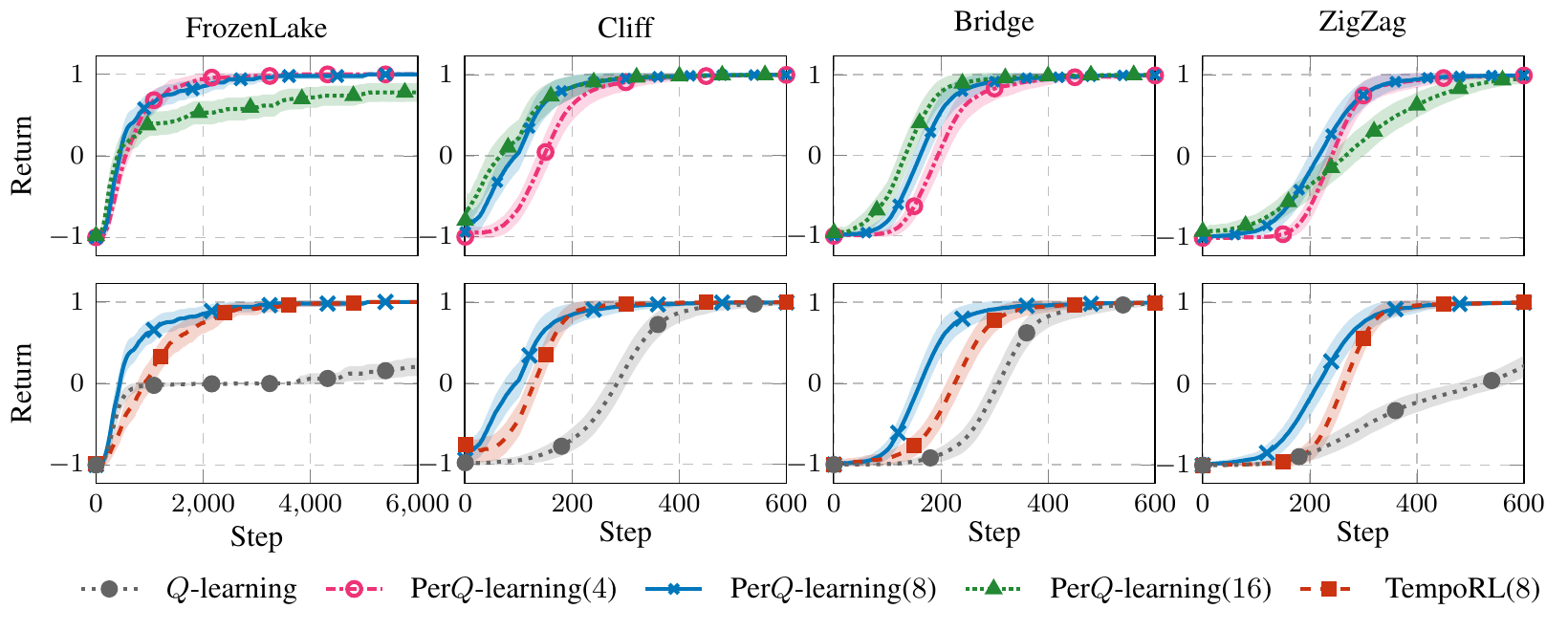}
\caption{Results on tabular environments. Top row: performances with different maximum persistences. In the legend, parenthesis denote the selected $K_{\max}$. Bottom row: \algnameQL and TempoRL comparison, $K_{\max}=8$.  50 runs (avg$\pm$ 95\% c.i.). }\label{fig:tabular_perf}\vspace*{-3mm}
\end{figure*}

\paragraph{Sample Complexity}
The second relevant effect of persistence concerns with the sample complexity. The intuition behind persistence relies on the fact that the most relevant information propagates faster through the state-action space, thanks to multi-step updates. Moreover, these updates are associated to a lower discount factor, for which it is possible to obtain better convergence rates, as proved in \cite{metelli2020control}, in which the sample complexity in a $k-$persistent MDP is reduced by a factor $(1-\gamma^k)/(1-\gamma)>1$.
In order to evaluate the sample efficiency of \algnameQL, separately from its effects on exploration, we considered a \textit{synchronous} setting \cite{kearns1999finite,sidford2018near} in a deterministic 6x6 Gridworld. At each iteration $t$, the agent has access to a set of independent samples for each state-action pair. 
In standard $Q$-learning, for each $(s,a)\in\SAs$, $Q(s,a)$ is updated. In \algnameQL, the samples are combined to obtain each possible set of $\overline{\kappa}$-persistent transitions, i.e., the tuples related to each possible $(s,a,k)\in\Sspace\times\mathcal{O}$, with $K_{\max}=6$; finally, the persistent $Q$ function is updated. 

In Figure~\ref{fig:sync} left, we compare the $L_{\infty}$ error of $Q$-learning estimating $Q^\star(s,a)$, \ie $ \max_{s,a \in \SAs }|Q_t(s,a) - Q^\star(s,a)|$, and that of \algnameQL estimating $\QstarK(s,a,k)$, \ie $\max_{s,a,k \in \Ss \times \mathcal{O}}|Q_t(s,a,k) - \QstarK(s,a,k)|$, as a function of the number of iterations $t$. We observe that, although estimating a higher-dimensional function (as $\QstarK(s,a,k)$ is a function of the persistence $k$ too), \algnameQL converges faster than $Q$-learning. In Figure~\ref{fig:sync} right, we plot the $L_{\infty}$ error experienced by \algnameQL for the different persistence options $\mathcal{O}^{(k)}$, \ie $\text{Error}_{\infty}(k) \coloneqq \max_{s,a \in \SAs}|Q_t(s,a,k)-Q^{\star}(s,a,k)|$ for $k \in \Kspace$. As expected, higher values of $k$ lead to faster convergence; consequently, the persistent Bellman operator helps improving the estimations also for the lower option sets. Indeed, we can see that also $Q_t(\cdot,\cdot,1)$, the primitive actions $Q$-function, converges faster than classic $Q$-learning (details in Appendix \ref{sec:apx sync}). 
\section{Persistent Deep Networks}\label{sec:perdq}
In this section, 
we develop the extension of \algnameQL to high-dimensional settings. Deep RL methods are becoming of fundamental importance when learning on real systems, as well as the research of methods to improve exploration and learning speed. It is straightforward to exploit Deep $Q$-Networks (DQN \cite{mnih2013playing, mnih2015human}) for learning in the options space $\mathcal{O}$. Standard DQN is augmented with $K_{\max}$ distinct sets of action outputs, to represent $Q$-value of the options space $\mathcal{O} = \mathcal{A \times K}$, while the first layers are shared, similarly to previous works~\cite{arulkumaran2016classifying, lakshminarayanan2017dynamic, Biedenkapp2021TempoRL}.
The resulting algorithm, \emph{\algnameExtDQN} (\algnameDQN) is obtained by exploiting the application of the empirical all-persistence Bellman operator. The main differences between \algnameDQN and standard DQN consist in: (i) a modified $\epsilon$-greedy strategy, which is equivalent to the one described for its tabular version; (ii) the use of \textit{multiple} replay buffers accounting for persistence.

\paragraph{Persistence Replay Buffers}
Whenever an option $o_t = (a_t, \overline{\kappa}_t)$ is executed, the partial history $H_t^{\overline{\kappa}_t}$ is decomposed in all its sub-transitions, which are used to update $Q$-values at any persistence, as shown in Section \ref{sec:perq}. The sub-transitions are stored in multiple replay buffers $\mathcal{D}_k$, one for each persistence $k\in\Kspace$. Specifically, $\mathcal{D}_k$ stores tuples in the form $(s,a_t,s',r, \overline{\kappa})$, as summarized in Algorithm \ref{alg:split_traj}, where $s$ and $s'$ are the first and the last state of the sub-transition, $r$ is the $\overline{\kappa}$-persistent reward, and $\overline{\kappa}$ is the true length of the sub-transition, which will then be used to suitably apply $\widehat{\mathcal{H}}_t^{\overline{\kappa}}$. 

Finally, the gradient update is computed by sampling a mini-batch of experience tuples from each replay buffer $\mathcal{D}_k$, in equal proportion.
Given the current network and target parametrizations $\vtheta$ and ${\vtheta}^{-}$, the temporal difference error of a sample $(s,a,r, s', \overline{\kappa})$ is computed as $ \widehat{\mathcal{H}}^{\overline{\kappa}} Q_{{\vtheta}^{-}}(s,a,k)- Q_{\vtheta}(s,a,k)$.
Our approach differs from TempoRL DQN \cite{Biedenkapp2021TempoRL}, which uses a dedicated network to learn the persistence at each state and employs a standard replay buffer, ignoring the persistence at which samples have been collected. 
\section{Experimental Evaluation}\label{sec:exps}
In this section, we show the empirical analysis of our approach on both the tabular setting (\algnameQL) and the function approximation one (\algnameDQN).

\paragraph{\algnameQL}
We present the results on the experiments in tabular environments, particularly suited for testing \algnameQL because of the sparsity of rewards. We start with the deterministic 6x10 grid-worlds  introduced by \citealt{Biedenkapp2021TempoRL}. In these environments, the episode ends if either the goal or a hole is reached, with $+1$ or $-1$ points respectively. In all the other cases, the reward is 0, and the episode continues (details in Appendix \ref{subsec:tabular_details}). Moreover, we experiment the 16x16 FrozenLake, from OpenAI Gym benchmark \cite{brockman2016open}, with rewards and transition process analogous to the previous case, but with randomly generated holes at the beginning of the episode.
\begin{figure*}[t]
\centering
\includegraphics[width=.95\textwidth]{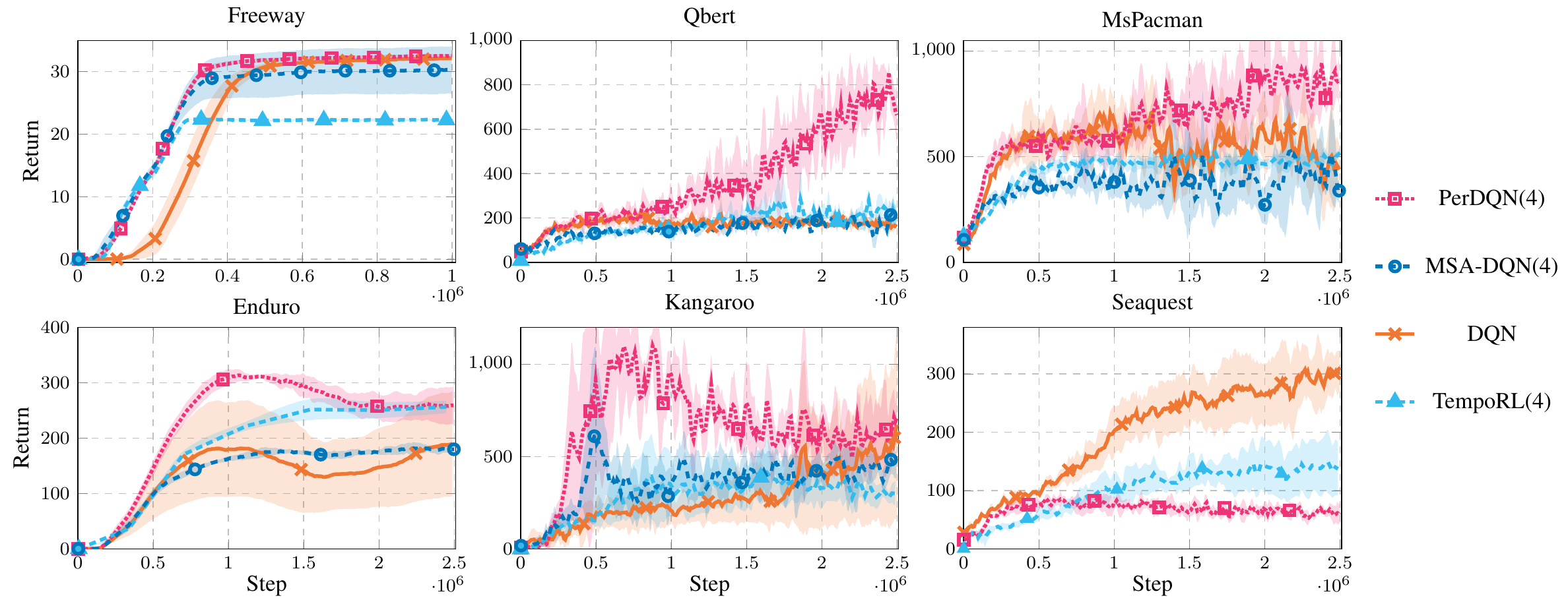}
\caption{Atari games results. Parenthesis in the legend denote the maximum persistence $K_{\max}$. 5 runs (avg$\pm$ 95\% c.i.). }\label{fig:atari}
\end{figure*}

The results are shown in Figure \ref{fig:tabular_perf}. In the top row, we compare the results on the performance when applying \algnameQL with different $K_{\max} \in\{4,8,16\}$. We can detect a faster convergence when passing from $K_{\max}=4$ to $8$. However, the largest value of $K_{\max}$ is not always the best one: while Bridge and Cliff show a slight improvement, performances in ZigZag and FrozenLake degrade. This is probably due to the nature of the environment. When there are many obstacles, high persistences might be inefficient, as the agent can get stuck or reach holes more easily. 
In the bottom plots of Figure \ref{fig:tabular_perf} we selected the results with $K_{\max}=8$, and compared them with TempoRL (with the same maximum \textit{skip-length} $J=8$) and classic Q-learning. In all cases, \algnameQL outperforms the other methods, especially Q-learning, whose convergence is significantly slower. 
exploration is very small. With this maximum persistence value \algnameQL outperforms TempoRL.
Further experiments with different values of $K_{\max}$ have been reported in Appendix \ref{sec:further_tab}, for space constraints.
In general, \algnameQL shows faster rates of improvements than TempoRL, especially in the first learning iterations. 
However, this advantage may not be consistent for every value of $K_{\max}$, and every environment, as also shown in Appendix \ref{sec:further_tab}. 

\paragraph{\algnameDQN}
Our implementation of \algnameDQN is based on OpenAI Gym \cite{brockman2016open} and Baselines \cite{baselines} Python toolkits. We start with MountainCar \cite{moore1990efficient}, as it is perhaps the most suited to evaluate the performance of persistence options. 
As shown in \cite{metelli2020control}, 1-step explorative policies usually fail to reach the goal.
 Figure \ref{fig:mtcar} shows that TempoRL and DQN cannot converge to the optimal policy, as already noticed in \cite{Biedenkapp2021TempoRL}, while \algnameDQN attains the optimal solution, that reaches the top of the mountain with the minimum loss. 

The algorithm is then tested in the challenging framework of Atari 2600 games, where we want to validate that action persistence is beneficial to speed up the initial phases of learning also in large environments.

The same architecture from \cite{mnih2013playing}, suitably modified as in Section~\ref{sec:perdq}, is used for all environments. For a fair comparison with TempoRL and standard DQN, persistence is implemented on top of the \emph{frame skip}. Thus, a one-step transition corresponds to 4 frame skips.
In Figure \ref{fig:atari} we compare \algnameDQN with TempoRL and classic DQN. In five games out of six, our \algnameDQN displays a faster learning curve thanks to its ability of reusing experience, although in some cases (e.g. Kangaroo) \algnameDQN seems to inherit the same instability issues of DQN, we conjecture due to the overestimation bias \cite{van2016deep}.
In order to better understand which beneficial effects are provided by action persistence alone and which ones derive from the use of bootstrap operator, we run an ablation experiment on the same tasks removing the latter one. The resulting algorithm is then similar to the Deep RL version of MSA-Q-learning \cite{schoknecht2003reinforcement}, which we called MSA-DQN. The results show that \algnameDQN always dominates over its counterpart without bootstrap. 
The crucial importance of the bootstrap operator is confirmed also in
the MountainCar setting where removing this feature causes a performance decrease, making its score comparable to TempoRL (see Appendix \ref{sec:further_dqn}).
Finally, we notice that in Seaquest
persistence seems to be detrimental for learning, as DQN clearly outperforms \algnameDQN. In this task, agents have to choose either to move or to shoot some moving targets. Persisting the shooting action, thus, may force the agent to stay still for a long time, hitting nothing. 
A possible solution could consist in the introduction of \textit{interrupting persistence}, in a similar fashion to interrupting options \cite{sutton1999between, mankowitz2014time}, which is an interesting future research direction.
\section{Discussion and Conclusions}\label{sec:conc}
In this paper, we considered RL policies that implement action persistence, modeled as \emph{persistence options}, selecting a primitive action and its duration. We defined the \emph{all-persistence} Bellman operator, which allows for an effective use of the experience collected at any time scale, as action-value function estimates can be updated simultaneously on the whole persistence set. In particular, low persistences (and primitive actions) can be updated by splitting the samples in their sub-transitions; high persistences can instead be improved by \textit{bootstrap}, i.e. by estimating the partial missing information. After proving that the new operator is a contraction, we extended classic $Q$-learning and DQN to their persistent version. The empirical analysis underlines the benefits of the new operator for exploration and estimation. Furthermore, the experimental campaign on tabular and deep RL settings demonstrated the effectiveness of our approach and the importance of considering temporal extended actions, as well as some limitations. Future research directions include the introduction of persistence interruption and techniques to overcome the overestimation bias. Furthermore, one could investigate the us of the operator in the actor-critic framework to cope with continuous actions.

\clearpage
\bibliography{biblio}

\begin{thebibliography}{46}
\providecommand{\natexlab}[1]{#1}

\bibitem[{Amin et~al.(2021)Amin, Gomrokchi, Aboutalebi, Satija, and
  Precup}]{amin2020locally}
Amin, S.; Gomrokchi, M.; Aboutalebi, H.; Satija, H.; and Precup, D. 2021.
\newblock Locally Persistent Exploration in Continuous Control Tasks with
  Sparse Rewards.
\newblock In \emph{International Conference on Machine Learning}, 275--285.
  PMLR.

\bibitem[{Arulkumaran et~al.(2016)Arulkumaran, Dilokthanakul, Shanahan, and
  Bharath}]{arulkumaran2016classifying}
Arulkumaran, K.; Dilokthanakul, N.; Shanahan, M.; and Bharath, A.~A. 2016.
\newblock Classifying options for deep reinforcement learning.
\newblock \emph{arXiv preprint arXiv:1604.08153}.

\bibitem[{Baird(1994)}]{baird1994reinforcement}
Baird, L.~C. 1994.
\newblock Reinforcement learning in continuous time: Advantage updating.
\newblock In \emph{Proceedings of 1994 IEEE International Conference on Neural
  Networks (ICNN'94)}, volume~4, 2448--2453. IEEE.

\bibitem[{Bellemare et~al.(2013)Bellemare, Naddaf, Veness, and
  Bowling}]{bellemare2013arcade}
Bellemare, M.~G.; Naddaf, Y.; Veness, J.; and Bowling, M. 2013.
\newblock The arcade learning environment: An evaluation platform for general
  agents.
\newblock \emph{Journal of Artificial Intelligence Research}, 47: 253--279.

\bibitem[{Bertsekas and Shreve(1996)}]{bertsekas2004stochastic}
Bertsekas, D.; and Shreve, S.~E. 1996.
\newblock \emph{Stochastic optimal control: the discrete-time case}, volume~5.
\newblock Athena Scientific.

\bibitem[{Biedenkapp et~al.(2021)Biedenkapp, Rajan, Hutter, and
  Lindauer}]{Biedenkapp2021TempoRL}
Biedenkapp, A.; Rajan, R.; Hutter, F.; and Lindauer, M.~T. 2021.
\newblock TempoRL: Learning When to Act.
\newblock In \emph{ICML}.

\bibitem[{Braylan et~al.(2015)Braylan, Hollenbeck, Meyerson, and
  Miikkulainen}]{braylan2015frame}
Braylan, A.; Hollenbeck, M.; Meyerson, E.; and Miikkulainen, R. 2015.
\newblock Frame skip is a powerful parameter for learning to play atari.
\newblock In \emph{Workshops at the Twenty-Ninth AAAI Conference on Artificial
  Intelligence}.

\bibitem[{Brockman et~al.(2016)Brockman, Cheung, Pettersson, Schneider,
  Schulman, Tang, and Zaremba}]{brockman2016open}
Brockman, G.; Cheung, V.; Pettersson, L.; Schneider, J.; Schulman, J.; Tang,
  J.; and Zaremba, W. 2016.
\newblock Openai gym.
\newblock \emph{arXiv preprint arXiv:1606.01540}.

\bibitem[{Catral et~al.(2010)Catral, Kirkland, Neumann, and
  Sze}]{catral2010kemeny}
Catral, M.; Kirkland, S.~J.; Neumann, M.; and Sze, N.-S. 2010.
\newblock The Kemeny constant for finite homogeneous ergodic Markov chains.
\newblock \emph{Journal of Scientific Computing}, 45(1): 151--166.

\bibitem[{Dabney, Ostrovski, and Barreto(2020)}]{dabney2020temporally}
Dabney, W.; Ostrovski, G.; and Barreto, A. 2020.
\newblock Temporally-Extended $\epsilon$-Greedy Exploration.
\newblock In \emph{International Conference on Learning Representations
  (ICLR)}.

\bibitem[{Dhariwal et~al.(2017)Dhariwal, Hesse, Klimov, Nichol, Plappert,
  Radford, Schulman, Sidor, Wu, and Zhokhov}]{baselines}
Dhariwal, P.; Hesse, C.; Klimov, O.; Nichol, A.; Plappert, M.; Radford, A.;
  Schulman, J.; Sidor, S.; Wu, Y.; and Zhokhov, P. 2017.
\newblock OpenAI Baselines.
\newblock \url{https://github.com/openai/baselines}.

\bibitem[{G.~Bellemare et~al.(2016)G.~Bellemare, Ostrovski, Guez, Thomas, and
  Munos}]{bellemare2016increasing}
G.~Bellemare, M.; Ostrovski, G.; Guez, A.; Thomas, P.; and Munos, R. 2016.
\newblock Increasing the Action Gap: New Operators for Reinforcement Learning.
\newblock \emph{Proceedings of the AAAI Conference on Artificial Intelligence},
  30(1).

\bibitem[{Grigsby, Yoo, and Qi(2021)}]{grigsby2021towards}
Grigsby, J.; Yoo, J.~Y.; and Qi, Y. 2021.
\newblock Towards Automatic Actor-Critic Solutions to Continuous Control.
\newblock In \emph{Deep RL Workshop NeurIPS 2021}.

\bibitem[{Gu et~al.(2017)Gu, Holly, Lillicrap, and Levine}]{gu2017deep}
Gu, S.; Holly, E.; Lillicrap, T.; and Levine, S. 2017.
\newblock Deep reinforcement learning for robotic manipulation with
  asynchronous off-policy updates.
\newblock In \emph{2017 IEEE international conference on robotics and
  automation (ICRA)}, 3389--3396. IEEE.

\bibitem[{Haarnoja et~al.(2019)Haarnoja, Ha, Zhou, Tan, Tucker, and
  Levine}]{haarnoja2019learning}
Haarnoja, T.; Ha, S.; Zhou, A.; Tan, J.; Tucker, G.; and Levine, S. 2019.
\newblock Learning to Walk Via Deep Reinforcement Learning.
\newblock In Bicchi, A.; Kress{-}Gazit, H.; and Hutchinson, S., eds.,
  \emph{Robotics: Science and Systems XV}.

\bibitem[{Kalyanakrishnan et~al.(2021)Kalyanakrishnan, Aravindan, Bagdawat,
  Bhatt, Goka, Gupta, Krishna, and Piratla}]{kalyanakrishnan2021analysis}
Kalyanakrishnan, S.; Aravindan, S.; Bagdawat, V.; Bhatt, V.; Goka, H.; Gupta,
  A.; Krishna, K.; and Piratla, V. 2021.
\newblock An Analysis of Frame-skipping in Reinforcement Learning.
\newblock \emph{arXiv preprint arXiv:2102.03718}.

\bibitem[{Kearns and Singh(1999)}]{kearns1999finite}
Kearns, M.; and Singh, S. 1999.
\newblock Finite-sample convergence rates for Q-learning and indirect
  algorithms.
\newblock \emph{Advances in Neural Information Processing Systems (NIPS)},
  996--1002.

\bibitem[{Khan et~al.(2019)Khan, Feng, Liu, and Asghar}]{khan2019optimal}
Khan, A.; Feng, J.; Liu, S.; and Asghar, M.~Z. 2019.
\newblock Optimal skipping rates: training agents with fine-grained control
  using deep reinforcement learning.
\newblock \emph{Journal of Robotics}, 2019.

\bibitem[{Kilinc, Hu, and Montana(2019)}]{kilinc2019reinforcement}
Kilinc, O.; Hu, Y.; and Montana, G. 2019.
\newblock Reinforcement Learning for Robotic Manipulation using Simulated
  Locomotion Demonstrations.
\newblock \emph{CoRR}, abs/1910.07294.

\bibitem[{Kiran et~al.(2021)Kiran, Sobh, Talpaert, Mannion, Al~Sallab,
  Yogamani, and P{\'e}rez}]{kiran2021deep}
Kiran, B.~R.; Sobh, I.; Talpaert, V.; Mannion, P.; Al~Sallab, A.~A.; Yogamani,
  S.; and P{\'e}rez, P. 2021.
\newblock Deep reinforcement learning for autonomous driving: A survey.
\newblock \emph{IEEE Transactions on Intelligent Transportation Systems}.

\bibitem[{Kirkland(2010)}]{kirkland2010fastest}
Kirkland, S. 2010.
\newblock Fastest expected time to mixing for a Markov chain on a directed
  graph.
\newblock \emph{Linear Algebra and its Applications}, 433(11-12): 1988--1996.

\bibitem[{Kober and Peters(2014)}]{kober2013learning}
Kober, J.; and Peters, J. 2014.
\newblock \emph{Learning Motor Skills - From Algorithms to Robot Experiments},
  volume~97 of \emph{Springer Tracts in Advanced Robotics}.
\newblock Springer.
\newblock ISBN 978-3-319-03193-4.

\bibitem[{Lakshminarayanan, Sharma, and
  Ravindran(2017)}]{lakshminarayanan2017dynamic}
Lakshminarayanan, A.~S.; Sharma, S.; and Ravindran, B. 2017.
\newblock Dynamic Action Repetition for Deep Reinforcement Learning.
\newblock In Singh, S.~P.; and Markovitch, S., eds., \emph{Proceedings of the
  Thirty-First {AAAI} Conference on Artificial Intelligence (AAAI)},
  2133--2139. {AAAI} Press.

\bibitem[{Lillicrap et~al.(2016)Lillicrap, Hunt, Pritzel, Heess, Erez, Tassa,
  Silver, and Wierstra}]{lillicrap2015continuous}
Lillicrap, T.~P.; Hunt, J.~J.; Pritzel, A.; Heess, N.; Erez, T.; Tassa, Y.;
  Silver, D.; and Wierstra, D. 2016.
\newblock Continuous control with deep reinforcement learning.
\newblock In Bengio, Y.; and LeCun, Y., eds., \emph{4th International
  Conference on Learning Representations ({ICLR})}.

\bibitem[{Mankowitz, Mann, and Mannor(2014)}]{mankowitz2014time}
Mankowitz, D.~J.; Mann, T.~A.; and Mannor, S. 2014.
\newblock Time regularized interrupting options.
\newblock In \emph{Internation Conference on Machine Learning (ICML)}.

\bibitem[{Mann, Mannor, and Precup(2015)}]{mann2015approximate}
Mann, T.~A.; Mannor, S.; and Precup, D. 2015.
\newblock Approximate Value Iteration with Temporally Extended Actions.
\newblock \emph{J. Artif. Intell. Res.}, 53: 375--438.

\bibitem[{Metelli et~al.(2020)Metelli, Mazzolini, Bisi, Sabbioni, and
  Restelli}]{metelli2020control}
Metelli, A.~M.; Mazzolini, F.; Bisi, L.; Sabbioni, L.; and Restelli, M. 2020.
\newblock Control frequency adaptation via action persistence in batch
  reinforcement learning.
\newblock In \emph{International Conference on Machine Learning (ICML)},
  6862--6873. PMLR.

\bibitem[{Metelli, Mutti, and Restelli(2018)}]{metelli2018configurable}
Metelli, A.~M.; Mutti, M.; and Restelli, M. 2018.
\newblock Configurable Markov Decision Processes.
\newblock In Dy, J.~G.; and Krause, A., eds., \emph{Proceedings of the 35th
  International Conference on Machine Learning (ICML)}, volume~80 of
  \emph{Proceedings of Machine Learning Research}, 3488--3497. {PMLR}.

\bibitem[{Mnih et~al.(2013)Mnih, Kavukcuoglu, Silver, Graves, Antonoglou,
  Wierstra, and Riedmiller}]{mnih2013playing}
Mnih, V.; Kavukcuoglu, K.; Silver, D.; Graves, A.; Antonoglou, I.; Wierstra,
  D.; and Riedmiller, M. 2013.
\newblock Playing atari with deep reinforcement learning.
\newblock \emph{arXiv preprint arXiv:1312.5602}.

\bibitem[{Mnih et~al.(2015)Mnih, Kavukcuoglu, Silver, Rusu, Veness, Bellemare,
  Graves, Riedmiller, Fidjeland, Ostrovski et~al.}]{mnih2015human}
Mnih, V.; Kavukcuoglu, K.; Silver, D.; Rusu, A.~A.; Veness, J.; Bellemare,
  M.~G.; Graves, A.; Riedmiller, M.; Fidjeland, A.~K.; Ostrovski, G.; et~al.
  2015.
\newblock Human-level control through deep reinforcement learning.
\newblock \emph{nature}, 518(7540): 529--533.

\bibitem[{Moore(1991)}]{moore1990efficient}
Moore, A.~W. 1991.
\newblock Efficient memory based learning for robot control.
\newblock \emph{PhD Thesis, Computer Laboratory, University of Cambridge}.

\bibitem[{Park, Kim, and Kim(2021)}]{park2021time}
Park, S.; Kim, J.; and Kim, G. 2021.
\newblock Time Discretization-Invariant Safe Action Repetition for Policy
  Gradient Methods.
\newblock \emph{Advances in Neural Information Processing Systems (NeurIPS)},
  34.

\bibitem[{Patel, Agharkar, and Bullo(2015)}]{patel2015robotic}
Patel, R.; Agharkar, P.; and Bullo, F. 2015.
\newblock Robotic surveillance and Markov chains with minimal weighted Kemeny
  constant.
\newblock \emph{IEEE Transactions on Automatic Control}, 60(12): 3156--3167.

\bibitem[{Precup(2001)}]{precup2001temporal}
Precup, D. 2001.
\newblock \emph{Temporal abstraction in reinforcement learning.}
\newblock Ph.D. thesis, University of Massachusetts Amherst.

\bibitem[{Puterman(2014)}]{puterman2014markov}
Puterman, M.~L. 2014.
\newblock \emph{Markov Decision Processes: Discrete Stochastic Dynamic
  Programming}.
\newblock John Wiley \& Sons.

\bibitem[{Schoknecht and Riedmiller(2003)}]{schoknecht2003reinforcement}
Schoknecht, R.; and Riedmiller, M. 2003.
\newblock Reinforcement learning on explicitly specified time scales.
\newblock \emph{Neural Computing \& Applications}, 12(2): 61--80.

\bibitem[{Schulman et~al.(2017)Schulman, Wolski, Dhariwal, Radford, and
  Klimov}]{schulman2017proximal}
Schulman, J.; Wolski, F.; Dhariwal, P.; Radford, A.; and Klimov, O. 2017.
\newblock Proximal Policy Optimization Algorithms.
\newblock \emph{CoRR}, abs/1707.06347.

\bibitem[{Sharma, Srinivas, and Ravindran(2017)}]{sharma2017learning}
Sharma, S.; Srinivas, A.; and Ravindran, B. 2017.
\newblock Learning to repeat: Fine grained action repetition for deep
  reinforcement learning.
\newblock \emph{arXiv preprint arXiv:1702.06054}.

\bibitem[{Sidford et~al.(2018)Sidford, Wang, Wu, Yang, and
  Ye}]{sidford2018near}
Sidford, A.; Wang, M.; Wu, X.; Yang, L.~F.; and Ye, Y. 2018.
\newblock Near-optimal time and sample complexities for solving Markov decision
  processes with a generative model.
\newblock In \emph{Proceedings of the 32nd International Conference on Neural
  Information Processing Systems}, 5192--5202.

\bibitem[{Singh et~al.(2000)Singh, Jaakkola, Littman, and
  Szepesv{\'a}ri}]{singh2000convergence}
Singh, S.; Jaakkola, T.; Littman, M.~L.; and Szepesv{\'a}ri, C. 2000.
\newblock Convergence results for single-step on-policy reinforcement-learning
  algorithms.
\newblock \emph{Machine learning}, 38(3): 287--308.

\bibitem[{Sutton and Barto(2018)}]{sutton2018reinforcement}
Sutton, R.~S.; and Barto, A.~G. 2018.
\newblock \emph{Reinforcement learning: An introduction}.
\newblock MIT press.

\bibitem[{Sutton, Precup, and Singh(1999)}]{sutton1999between}
Sutton, R.~S.; Precup, D.; and Singh, S.~P. 1999.
\newblock Between MDPs and Semi-MDPs: {A} Framework for Temporal Abstraction in
  Reinforcement Learning.
\newblock \emph{Artif. Intell.}, 112(1-2): 181--211.

\bibitem[{Tallec, Blier, and Ollivier(2019)}]{tallec2019time}
Tallec, C.; Blier, L.; and Ollivier, Y. 2019.
\newblock Making Deep Q-learning methods robust to time discretization.
\newblock In Chaudhuri, K.; and Salakhutdinov, R., eds., \emph{Proceedings of
  the 36th International Conference on Machine Learning (ICML)}, volume~97 of
  \emph{Proceedings of Machine Learning Research}, 6096--6104. {PMLR}.

\bibitem[{van Hasselt, Guez, and Silver(2016)}]{van2016deep}
van Hasselt, H.; Guez, A.; and Silver, D. 2016.
\newblock Deep Reinforcement Learning with Double Q-Learning.
\newblock \emph{Proceedings of the AAAI Conference on Artificial Intelligence},
  30(1).

\bibitem[{Watkins(1989)}]{watkins1989learning}
Watkins, C. J. C.~H. 1989.
\newblock \emph{Learning from delayed rewards}.
\newblock Ph.D. thesis, King's College, University of Cambridge.

\bibitem[{Yu, Xu, and Zhang(2021)}]{yu2021taac}
Yu, H.; Xu, W.; and Zhang, H. 2021.
\newblock TAAC: Temporally Abstract Actor-Critic for Continuous Control.
\newblock \emph{Advances in Neural Information Processing Systems (NeurIPS)},
  34.

\end{thebibliography}
\clearpage
\onecolumn
\appendix
\setlength{\abovedisplayskip}{8pt}
\setlength{\belowdisplayskip}{8pt}

\section{Proofs and Derivations}\label{sec:apx proof}

\begin{restatable}[Decomposition of $r_k$]{lemma}{}\label{lem:r_k dec}
Let $r_k(s,a)$ the expected $k$-persistent reward. Let $k'<k$, then it holds that:
$$ r_k(s,a) = r_{k'}(s,a) + \gamma^{k'}\int_\Sspace P_{k'}(\mathrm{d}s'|s,a)r_{k-k'}(s',a) \qquad \forall (s,a) \in \SAs.$$ 
\end{restatable}
\begin{proof}
From the definition of $r_k$, it holds that $r_k(s,a) = \sum_{i=0}^{k-1}\gamma^i\big((P^\delta)^ir\big)(s,a)$. Hence:
\begin{align*}
    r_k(s,a) &= \sum_{i=0}^{k-1}\gamma^i\big((P^\delta)^ir\big)(s,a) \\
    &= \sum_{i=0}^{k'-1}\gamma^i\big((P^\delta)^ir\big)(s,a) + \sum_{i=k'}^{k-1}\gamma^i\big((P^\delta)^ir\big)(s,a) \\
    &= r_{k'}(s,a) + \sum_{i=k'}^{k-1}\gamma^i\int_\Sspace P_i(\mathrm{d}s'|s,a)r(s',a) \\
    &= r_{k'}(s,a) + \sum_{i=k'}^{k-1}\gamma^i\int_\Sspace P_{k'}(\mathrm{d}s''|s,a) \int_\Sspace P_{i-k'}(\mathrm{d}s'|s'',a)r(s',a) \\
    &= r_{k'}(s,a) + \gamma^{k'}\int_\Sspace P_{k'}(\mathrm{d}s''|s,a) \sum_{j=0}^{k-k'-1}\gamma^j \int_\Sspace P_{j}(\mathrm{d}s'|s'',a)r(s',a) \\
     &= r_{k'}(s,a) + \gamma^{k'}\int_\Sspace P_{k'}(\mathrm{d}s''|s,a) r_{k-k'}(s'',a).
\end{align*}
\end{proof}

\HQcontraction*
\begin{proof}

(i)  First, we prove the contraction property:
we consider the $L_\infty$-norm applied to the state-action-persistence space $\SAKs$:
\begin{align*}
    \|\Hop^{\ksamp} & Q_1 -\Hop^{\ksamp} Q_2 \|_{\infty} = \sup_{s,a,k \in \SAKs}\left| \Hop^{\ksamp} Q_1(s,a,k)-\Hop^{\ksamp} Q_2(s,a,k) \right| \\
    &= \sup_{s,a,k \in \SAKs}\bigg|\mathbbm{1}_{k\le\ksamp}\left(\left(T^{\star}Q_1\right)(s,a,k)-\left(T^{\star}Q_2\right)(s,a,k) \right) \\ & \qquad \qquad \qquad
    +\mathbbm{1}_{k\le\ksamp}\left(\left(T^{\ksamp}Q_1\right)(s,a,k)-\left(T^{\ksamp}Q_2\right)(s,a,k) \right)\bigg| \\
    & = \sup_{s,a,k \in\SAKs}\bigg|\gamma^k\mathbbm{1}_{k\le\ksamp}\int_\Sspace \Pkint\left[\sup_{a',k' \in \mathcal{A\times K}} Q_1(s',a',k')-\sup_{a',k'\in \mathcal{A\times K}}Q_2(s',a',k')\right] \\
    & \qquad +\gamma^{\ksamp}\mathbbm{1}_{k>\ksamp}\int_\Sspace P_{\ksamp}(\mathrm{d}s'|s,a)\left[Q_1(s',a,k-\ksamp)-Q_2(s',a,k-\ksamp)\right]\bigg|\\
    & \le \sup_{s,a,k\in\SAKs}\bigg\{\gamma^k\mathbbm{1}_{k\le\ksamp}\int_\Sspace \Pkint\bigg|\sup_{a',k'\in \mathcal{A\times K}}Q_1(s',a',k')-\sup_{a',k'\in \mathcal{A\times K}}Q_2(s',a',k')\bigg| \\
    & \qquad + \gamma^{\ksamp}\mathbbm{1}_{k>\ksamp}\int_\Sspace P_{\ksamp}(\mathrm{d}s'|s,a)\bigg|Q_1(s',a,k-\ksamp)-Q_2(s',a,k-\ksamp)\bigg|\bigg\}\\
    & \le \sup_{s,a,k\in\SAKs}\bigg\{\gamma^k\mathbbm{1}_{k\le\ksamp}\int_\Sspace \Pkint\sup_{\widetilde{a},\widetilde{k}\in \mathcal{A\times K}}\bigg|Q_1(s',\widetilde{a},\widetilde{k})-Q_2(s',\widetilde{a},\widetilde{k})\bigg| \\
    & \qquad + \gamma^{\ksamp}\mathbbm{1}_{k>\ksamp}\int_\Sspace\Psampint\sup_{\widetilde{s},\widetilde{a}\in \mathcal{A\times K}}\bigg|Q_1(\widetilde{s},\widetilde{a},k-\ksamp)-Q_2(\widetilde{s},\widetilde{a},k-\ksamp)\bigg|\bigg\}\\
    & \le \sup_{s,a,k\in\SAKs}\bigg\{\gamma^k\mathbbm{1}_{k\le\ksamp}\int_\Sspace \Pkint\sup_{\widetilde{s},\widetilde{a},\widetilde{k}\in\SAKs}\bigg|Q_1(\widetilde{s},\widetilde{a},\widetilde{k})-Q_2(\widetilde{s},\widetilde{a},\widetilde{k})\bigg| \\
    & \qquad +\gamma^{\ksamp}\mathbbm{1}_{k>\ksamp}\int_\Sspace\Psampint\sup_{\widetilde{s},\widetilde{a},\widetilde{k}\in\SAKs}\bigg|Q_1(\widetilde{s},\widetilde{a},\widetilde{k})-Q_2(\widetilde{s},\widetilde{a},\widetilde{k})\bigg|\bigg\}\\
    & \le \sup_{s,a,k\in\SAKs}\bigg\{\bigg(\gamma^k\mathbbm{1}_{k\le\ksamp} +\gamma^{\ksamp}\mathbbm{1}_{k>\ksamp}\bigg)\|Q_1-Q_2\|_\infty\bigg\}\\
    & = \|Q_1-Q_2\|_\infty \sup_{k \in \Kspace} \gamma^{\min\{k,\ksamp\}} =\gamma\|Q_1-Q_2\|_\infty.
\end{align*}

(ii): Since the contraction property holds, and being $(\SAKs, d_\infty)$ a complete metric space (with $d_\infty$ being the distance induced by $L_\infty$ norm), the Banach Fixed-point theorem holds, guaranteeing convergence to a unique fixed point. 

We now show that $\QstarK$ is a fixed point of $T^{\star}$. We first need to define the extended Bellman expectation operators $\Tpsi: \mathscr{B}(\Sspace \times \mathcal{O}) \rightarrow \mathscr{B}(\Sspace \times \mathcal{O})$ with $f \in \mathscr{B}(\Sspace \times \mathcal{O})$ and $\psi\in\Psi$:\footnote{The following can be extended without loss of generality to a continuous action space.}
\begin{align*}
    (\Tpsi f)(s,a,k) &= r_k(s,a)+\gamma^k\int_\Sspace \PK(\mathrm{d}s'|s,a) V(s'), \\
    V(s) &=\sum_{(a,k)\in\mathcal{O}} \psi(a,k|s)f(s,a,k)
\end{align*}
As with standard Bellman operators, it trivially holds that $\Tpsi Q^\psi=Q^\psi\ \forall\psi\in\Psi$.
Thus, we can take into account the definition of value function $V^\psi$ of a policy $\psi$ and the standard Bellman Equations:
\begin{equation}
\begin{aligned}
    Q^\psi(s,a,k) &= r_k(s,a)+\gamma^k \int_\Sspace \Pkint V^\psi(s') \\
    V^\psi(s)& = \sum_{(a,k)\in\mathcal{O}} \psi(a,k|s)Q^\psi(s,a,k)
\end{aligned}
\end{equation}
Following the same argument as in \cite{puterman2014markov}, it holds that the optimal operator $T^{\star}$ improves the action-value function, i.e. $T^\star Q^\psi \ge Q^\psi$, and consequently $\QstarK$ is the (unique) fixed point for $T^{\star}$, i.e., $T^{\star}\QstarK=Q^{\star}$ by contraction mapping theorem.

Moreover, it holds that $T^{\ksamp}Q^\psi = Q^\psi$:
\begin{align}
    &(T^{\ksamp}Q^\psi)(s,a,k) = r_{\ksamp}(s,a) + \gamma^{\ksamp}\int_\Sspace P_{\ksamp}(\mathrm{d}s'|s,a) Q^\psi(s',a,k-\ksamp) \nonumber\\
    &= r_{\ksamp}(s,a) + \gamma^{\ksamp}\int_\Sspace \Pkint\bigg[r_{k-\ksamp}(s',a)+\gamma^{k-\ksamp}\int_{\Sspace}P_{k-\ksamp}(\mathrm{d}s''|s',a)V^\psi(s')\bigg]\nonumber\\
    &= \underbrace{r_{\ksamp}(s,a) + \gamma^{\ksamp}\int_\Sspace \Pkint r_{k-\ksamp}(s',a)}_{r_k(s,a)} + \gamma^k \int_\Sspace P_{\ksamp}(\mathrm{d}s'|s,a)\int_\Sspace P_{k-\ksamp}(\mathrm{d}s''|s',a)V^\psi(s'')\label{r_dec}\\
    &= r_k(s,a) + \gamma^k \int_\Sspace \Pkint V^\psi(s')=Q^\psi(s,a,k)\nonumber,
\end{align}
where in Equation \eqref{r_dec} we used Lemma \ref{lem:r_k dec}.

In conclusion,
\begin{align*}
    \Hop^{\ksamp}\QstarK &= \bigg( \mathbbm{1}_{k\le\ksamp}T^{\star}+\mathbbm{1}_{k>\ksamp}T^{\ksamp}\bigg)\QstarK \\
    &=\mathbbm{1}_{k\le\ksamp}T^{\star}\QstarK+\mathbbm{1}_{k>\ksamp}T^{\ksamp}\QstarK\\
    &=\mathbbm{1}_{k\le\ksamp}\QstarK+\mathbbm{1}_{k>\ksamp}\QstarK=\QstarK.
\end{align*}

(iii) We provide the proof of monotonic property: given $(s,a)\in\SAs$, and given $k\le k'$, we have:
\begin{align*}
    \QstarK(s,a,k) &= (T^{\star}\QstarK)(s,a,k) \\
    &= r_k(s,a) +\gamma^k \int_\Sspace P_k(\mathrm{d}s'|s,a)\max_{\overline{a},\overline{k} \in \mathcal{A \times K}}\QstarK(s',\overline{a},\overline{k} )\\
    &\ge r_k(s,a) +\gamma^k \int_\Sspace P_k(\mathrm{d}s'|s,a)\QstarK(s',a,k'-k)\\
    &= T^k\QstarK(s,a,k')=\QstarK(s,a,k').
\end{align*}
\end{proof}

\begin{restatable}[Equivalence of $\QstarK$ and $Q^\star$]{coroll}{}\label{cor:apx_Qstar_equiv}
For all $(s,a)\in\SAs$, the optimal action-value function $Q^\star$ and the optimal option-value function restricted to the primitive actions coincide, i.e. $$ \QstarK(s,a,1) = Q^\star(s,a)$$
\end{restatable}
\begin{proof}
Trivially, $\QstarK$, defined on $\Psi$, coincides with the classic $Q^\star$ defined on the \textit{primitive} policies $\pi\in\Pi$: in a first instance, we remark that $\Pi \subset \Psi$, hence all the policies defined on the space of primitive actions belong to the set of persistent policies. Furthermore, we can consider property (iii) of Theorem \ref{prop:HQ_contraction}. As a consequence $\QstarK(s,a,1)\ge Q^{\star}(s,a,k) \forall k\in\Kspace$, i.e., for each state $s\in\Sspace$ there is at least one optimal primitive action $a\in\Aspace$ which is optimal among all the option set $\mathcal{O}$. Consequently, the two optimal action-value functions coincide.
\end{proof}

\clearpage
\section{Further clarifications}\label{sec:apx scheme}
\subsection{\algnameQL Update scheme}
In Figure \ref{fig:all_pers_update} we present a scheme that shows an examples of the update order for Algorithm \ref{alg:PerQ}.

\begin{figure}[H]
\centering
\includegraphics[width=.5\columnwidth]{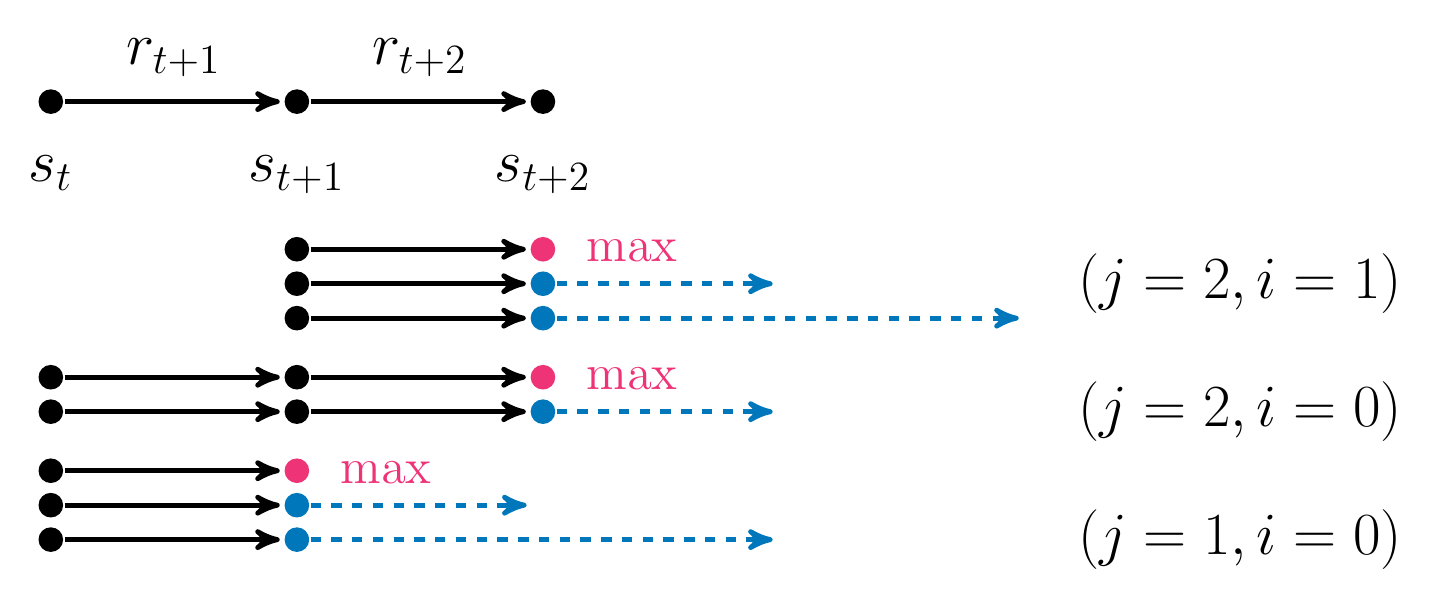}
\caption{An example of the update order for Algorithm \ref{alg:pers_update} with $\ksamp_t=2$ and $K_{\max}=3$. 
Applications of $\widehat{T}^\star_t$ and $\widehat{T}^{\ksamp_t}_t$ are denoted, respectively, by magenta and blue nodes, while dashed arrows represent the bootstrap persistence.}
\label{fig:all_pers_update}
\end{figure}
\clearpage
\section{Details on Persistent Markov Chains and Kemeny constant}\label{sec:apx kemeny}
In this section, we provide more details regarding how to obtain the transition Kernel implied from a persistent random variable acting on the environment, and how to compute its Kemeny's constant.

Consider an agent where actions are sampled from a generic policy $\pi(\cdot|s)=:\pi$ on the action space $\Aspace$, independent from the current state, and where persistence is sampled from a discrete distribution $\omega$  with support in  $\{1,\dots, K_{\max}\}$, independent from $\pi$. They constitute the policy $\psi$ over persistence options. The $k$-step Transition Chain induced by $\pi$ over the state space $\Sspace$ is defined as  $P_k^\pi(s'|s) = \int_{\Aspace}P_k(s'|s,a)\pi(da|s)$. This is equivalent to the Markov chain induced by $\pi$ in the $k$-step MDP, where the control frequency is set to $k$ times the base duration $\delta$. 
We now consider the transition probability induced by the joint probability distributions $\pi$ and $\omega$ up to a maximum of $K_{\max}$ steps, which for simplicity will here be referred as $K$.
In order to define it, it is necessary to consider a fixed horizon $H$: when the total number of steps in the trajectory reaches the horizon, then the (eventual) persistence is stopped. This means that, if we start for example at the $H-j$ step, the probability of persisting $j$ times the sampled action is equivalent to $\sum_{i\ge j}\omega_i$. This assumption is necessary for the Markov condition to hold.
As a consequence, we define $\widetilde{\omega}_{j} = \big\{ \widetilde{\omega}_{i,j}\big\}_i$ as a reduced distribution of $\omega$ to $j$ steps:
\begin{equation*}
\widetilde{\omega}_{i,j}\coloneqq \begin{cases} \omega_i & \text{ if } i<j \\
\sum_{i=j}^{K} \omega_i& \text{ otherwise}
 \end{cases}.
\end{equation*}

Finally we can recursively define the transition probability in $H$ steps, induced by $\pi$ and $\omega$ as:
\begin{equation}\label{eq:P_pers}
    \mathbb{P}^{\pi,\omega}_{H} \coloneqq \sum_{k=1}^{K}\widetilde{\omega}_{k,H\wedge K}\mathbb{P}^{\pi, \widetilde{\omega}_{H-k}}_{H-k}P_k^{\pi},
\end{equation}
where $\mathbb{P}^{\pi,\omega}_0 = \mathbbm{1}_{\Sspace\times\Sspace}$ and $ a\wedge b=\min\{a,b\}$. 
Equation \eqref{eq:P_pers} is not trivial and needs some clarifications. Let's consider an example, where $K=4$ and $H=3$. In this case the persistence distribution is $\omega=\{\omega_1,\omega_2,\omega_3,\omega_4\}$. 
\begin{itemize}
    \item With probability $\omega_3+\omega_4$, the sampled persistence is equal to 3, and the related transition is $P_{3}^\pi$ (since $H=3$, sampling persistence 4 leads to repeat the action for 3 times);
    \item  With probability $\omega_2$, the sampled persistence is equal to 2. The first two steps evolve as $P^\pi_{2}$, and the last step follows $\mathbb{P}^{\pi,\widetilde{\omega}_1}_1=P^\pi$;
    \item  With probability $\omega_1$, the action is selected only once, and at the next step it has to be sampled again and eventually persisted for two steps w.p. $\omega_2+\omega_3+\omega_4$.
\end{itemize}   
In other terms, denoting $\widetilde{\omega}_1=\omega_1$, $\widetilde{\omega}_2=\omega_2$, and $\widetilde{\omega}_3 = \omega_3 + \omega_4$:
\begin{align*}
    \mathbb{P}^{\pi,\omega}_3 =& \widetilde{\omega}_3 P^\pi_3 + \widetilde{\omega}_2 P^\pi P^\pi_2 +\widetilde{\omega}_1 [\widetilde{\omega}_1 P^\pi P^\pi + (\widetilde{\omega}_2+\widetilde{\omega}_3) P^{\pi}_2] P^\pi\\
    =& \widetilde{\omega}_3 \mathbbm{1} P^\pi_3  + \widetilde{\omega}_2 \underbrace{[(\widetilde{\omega}_1 + \widetilde{\omega}_2+\widetilde{\omega}_3) P^\pi ]}_{=\mathbb{P}^{\pi, \widetilde{\omega}_1}_1} P^\pi_2 +\\
    &+ \widetilde{\omega}_1 \underbrace{[\widetilde{\omega}_1 (\widetilde{\omega}_1 + \widetilde{\omega}_2+\widetilde{\omega}_3)P^\pi P^\pi + (\widetilde{\omega}_1+
    \widetilde{\omega}_2) P^{\pi}_2]}_{=\mathbb{P}^{\pi, \widetilde{\omega}_2}_2} P^\pi_1 \\
    =& \widetilde{\omega}_3\mathbb{P}^{\pi,\omega}_0P_3^\pi + \widetilde{\omega}_2\mathbb{P}^{\pi,\omega}_1P_2^\pi+ \widetilde{\omega}_1\mathbb{P}^{\pi,\omega}_2P_1^\pi
\end{align*}
The meaning of the modified distribution $\widetilde{\omega}$ is related to the fact that, once the trajectory evolved for $k$ steps, the remaining $H-k$ are still sampled, but when the last step $H$ is reached, then the agent stops repeating in any case.

\paragraph{Kemeny's constant computation}
\begin{figure}[t]
\centering
\includegraphics[width=.8\columnwidth]{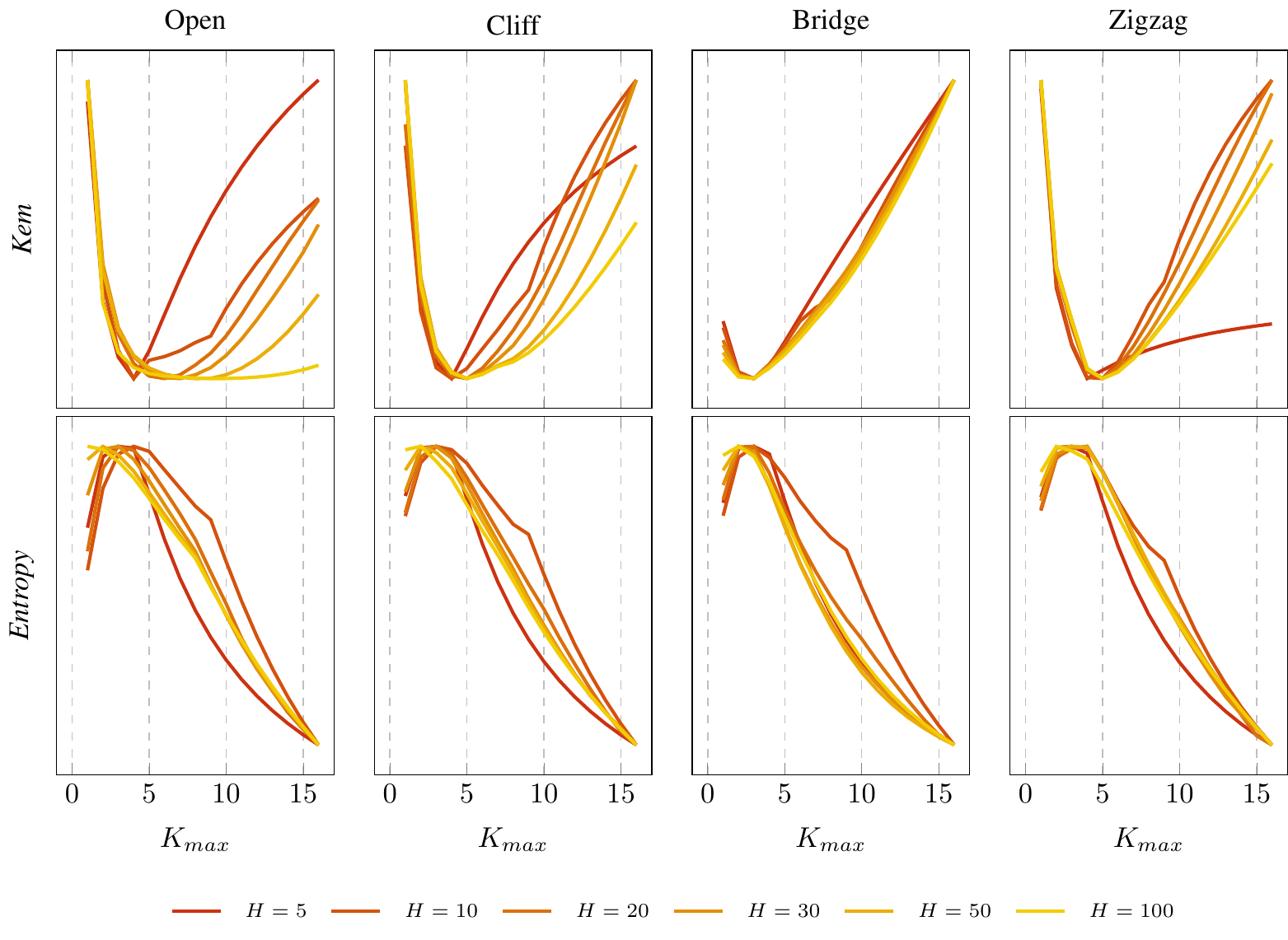}
\caption{Normalized Kemeny's constant and entropy in Tabular environments as function of the maximum persistence and horizon $H$.}\label{fig:kem_complete}
\end{figure}

The formula used to compute the Kemeny's constant from the transition Kernel $\mathbb{P}_H^{\pi,\omega}$ can be obtain thanks to the following Proposition \cite{kirkland2010fastest}.
\begin{restatable}[Kemeny's constant of an irreducible Markov Chain]{prop}{}\label{prop:kemeny}
Consider a Markov chain with an irreducible transition matrix P with eigenvalues $\lambda_1=1$ and $\lambda_i, i\in\{2,\dots,n\}$. The Kemeny constant of the Markov chain is given by
$$Kem = \sum_{i=2}^n \frac{1}{1-\lambda_i}. $$
\end{restatable}
The introduction of the parameter $H$ is necessary to retrieve an irreducible transition matrix  $\mathbb{P}^{\pi,\omega}_{H}$ maintaining the Markov property.

In order to compute the curves of Kemeny's constant in Figure \ref{fig:Kemeny}, we consider a $K_{max}$ as a variable, and exploration is performed by a discrete uniform random variable in $\mathcal{O}=\bigcup_{k\in\Kspace}\Ospace^{(k)}$ , i.e., the distribution $\pi$ is uniform over the action space $\Aspace=\{left,down, right, up\}$, and $\omega$ uniform over $\Kspace$.
In Figure \ref{fig:kem_complete} we show the curves of Kemeny's constant and entropy with different value of $K_{max}$ and $H$, and Figure \ref{fig:Kemeny} presented in Seciton \ref{sec:advantages} refers to the same Kemeny's curves selected for $H=30$.

As we can see in the figure, for each value of $H$ there is a similar pattern: increasing $K_{max}$, the related values for Kemeny's constant initially tend to decrease, indicating that persistence helps for a faster exploration through the state space. Persisting actions for long times does not help exploration, since agents might be more frequently standing in front of walls. Consequently, depending on the different design of the environments, Kemeny's values begin to increase. In the bottom plots of Figure \ref{fig:kem_complete} we can observe also the curves related to the entropy induced by $\mathbb{P}^{\pi,\omega}_{H}$: again, the maximum value of entropy is attained by $K_{max}>1$. However, the curves soon start to decrease dramatically, indicating that reaching distant states sooner is not strictly related to its visitation frequency. 
\section{Further Exploration Advantages: MountainCar}\label{sec:apx mtcar}
In this section, we provide further evidence regarding the advantages of exploration.
We study the effects of a persisted random policy on the MDP, \ie a policy $\psi \in \Psi$ over persistence options $\mathcal{O}$ in Mountain Car. 
We have collected all the states traversed by a full random agent with different values of $K_{max}$, both with a fixed persistence (Figure \ref{fig:adv_exploration_fix}) and a dynamic persistence (Figure \ref{fig:adv_exploration}). The figures represent heatmaps of the visitation frequency in the state space. As we can see, when $K_{max}$ is low the agent has less chance to reach the goal (represented by the blue dotted line). Increasing the persistence, the distribution over the states starts covering a wider region of the state space and reaching the goal with higher probability. 
\begin{figure}[t]
\centering
\includegraphics[width=.85\columnwidth]{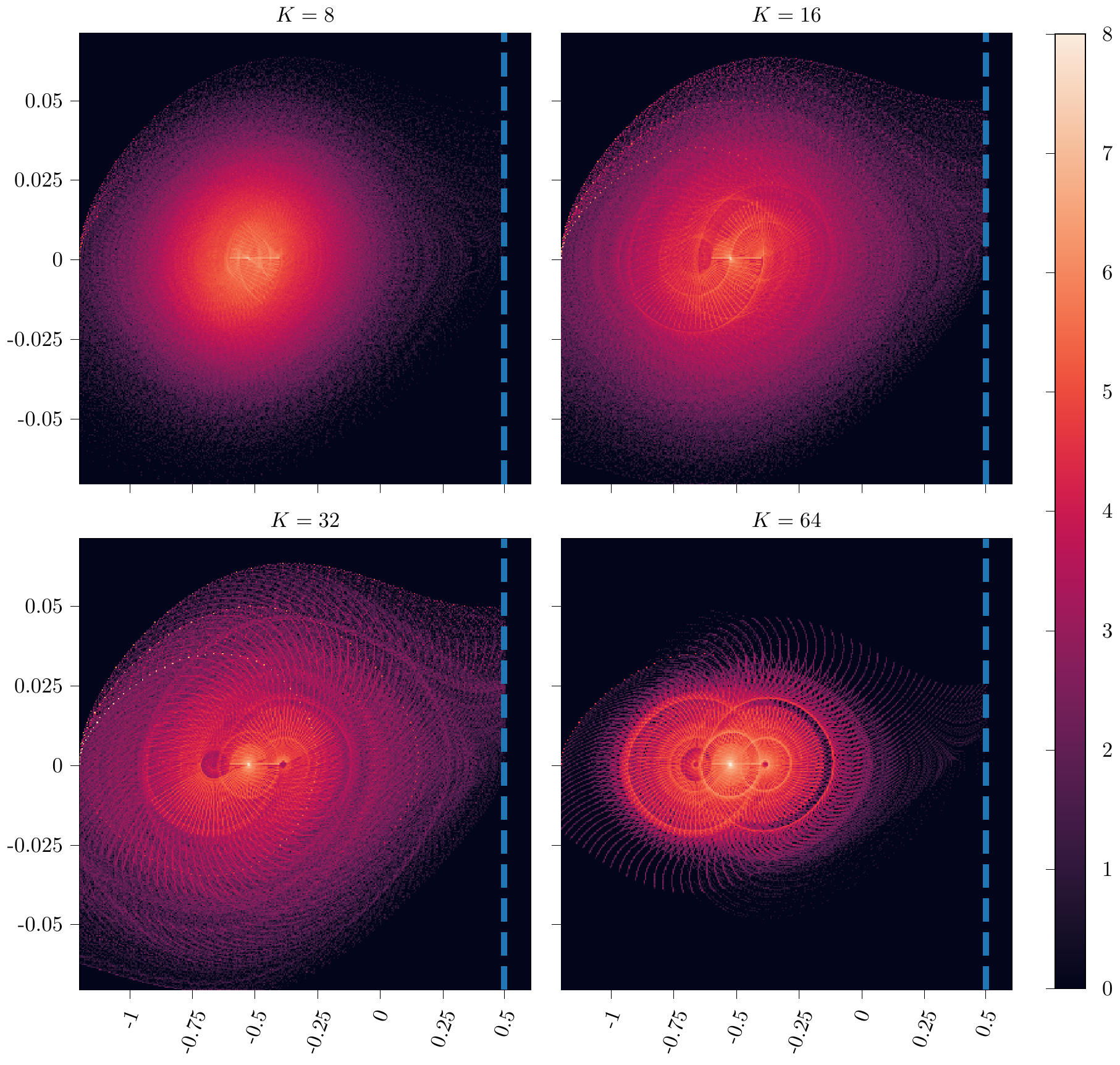}
\caption{Visitation frequencies in Mountain Car environment of different random policies with fixed $K$ persistences. The value represents in a logarithmic scale the counter of visited states. The $x$ and $y$ axes are respectively the position and the velocity of the car.  Blue dotted line represents the goal. 10.000 episodes}\label{fig:adv_exploration_fix}
\end{figure}

\begin{figure}[t]
\centering
\includegraphics[width=0.85\columnwidth]{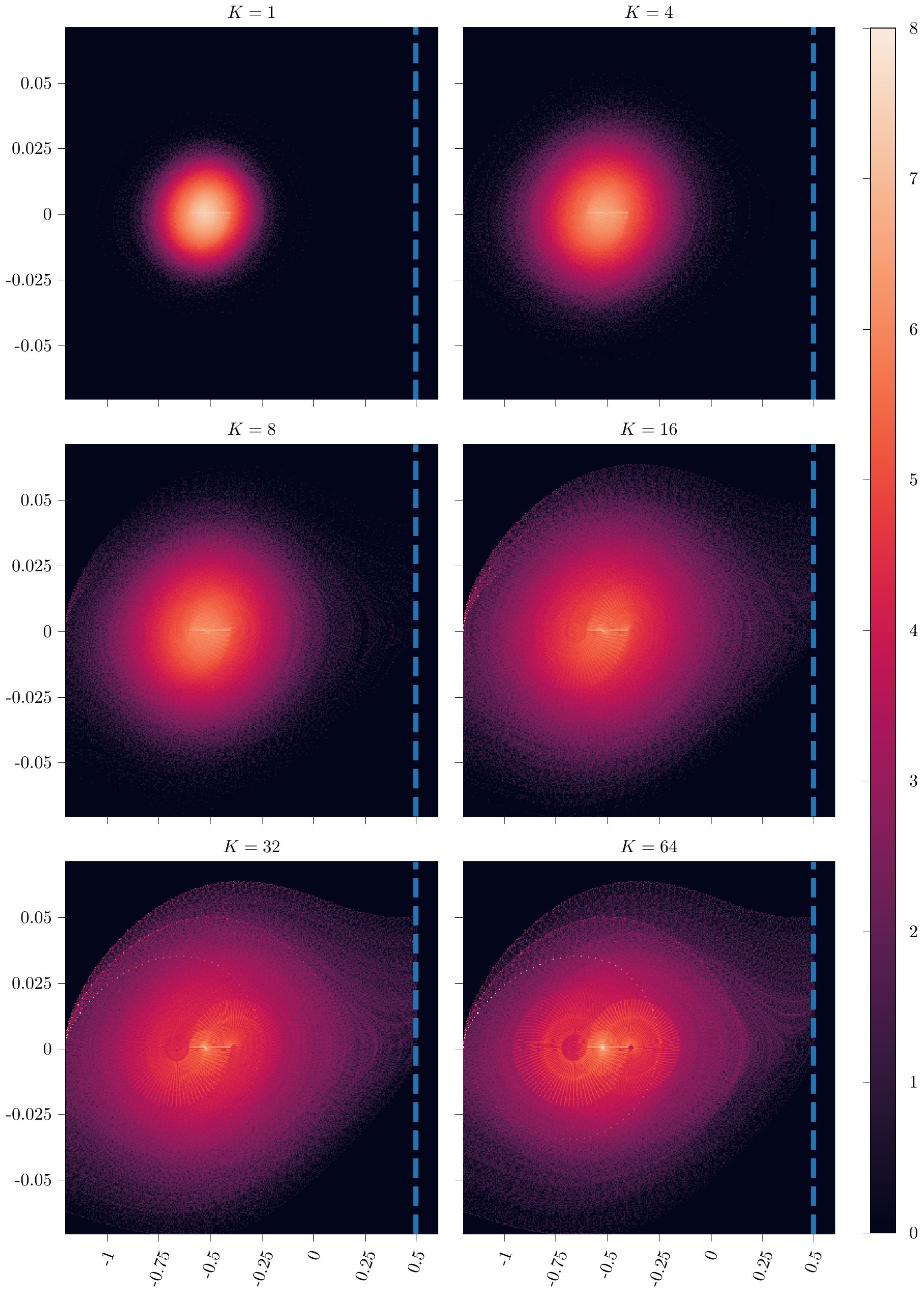}
\caption{Visitation frequencies in Mountain Car environment of different random persistence options. The value represents in a logarithmic scale the counter of visited states. The $x$ and $y$ axes are respectively the position and the velocity of the car. Blue dotted line represents the goal. 10.000 episodes}\label{fig:adv_exploration}
\end{figure}

We can observe that, even if the goal for some values of $K$ is almost equally visited, with a fixed persistence we have a different distribution of visited states. Moreover, a fixed, high persistence does not provide sufficient exploration to reach the goal, as we can see in figure \ref{fig:adv_exploration_fix} with $K=64$, especially if compared to the dynamic version in figure \ref{fig:adv_exploration}.
\clearpage
\section{Synchronous Update: additional Results}\label{sec:apx sync}
\begin{figure}
\centering
\includegraphics[width=0.16\columnwidth]{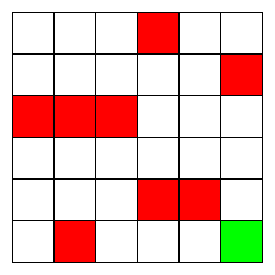}
\caption{Grid environment. Red cell denote holes, green cells the goal.}\label{fig:sync_env}
\end{figure}
\begin{figure}
\centering
\includegraphics[width=0.5\columnwidth]{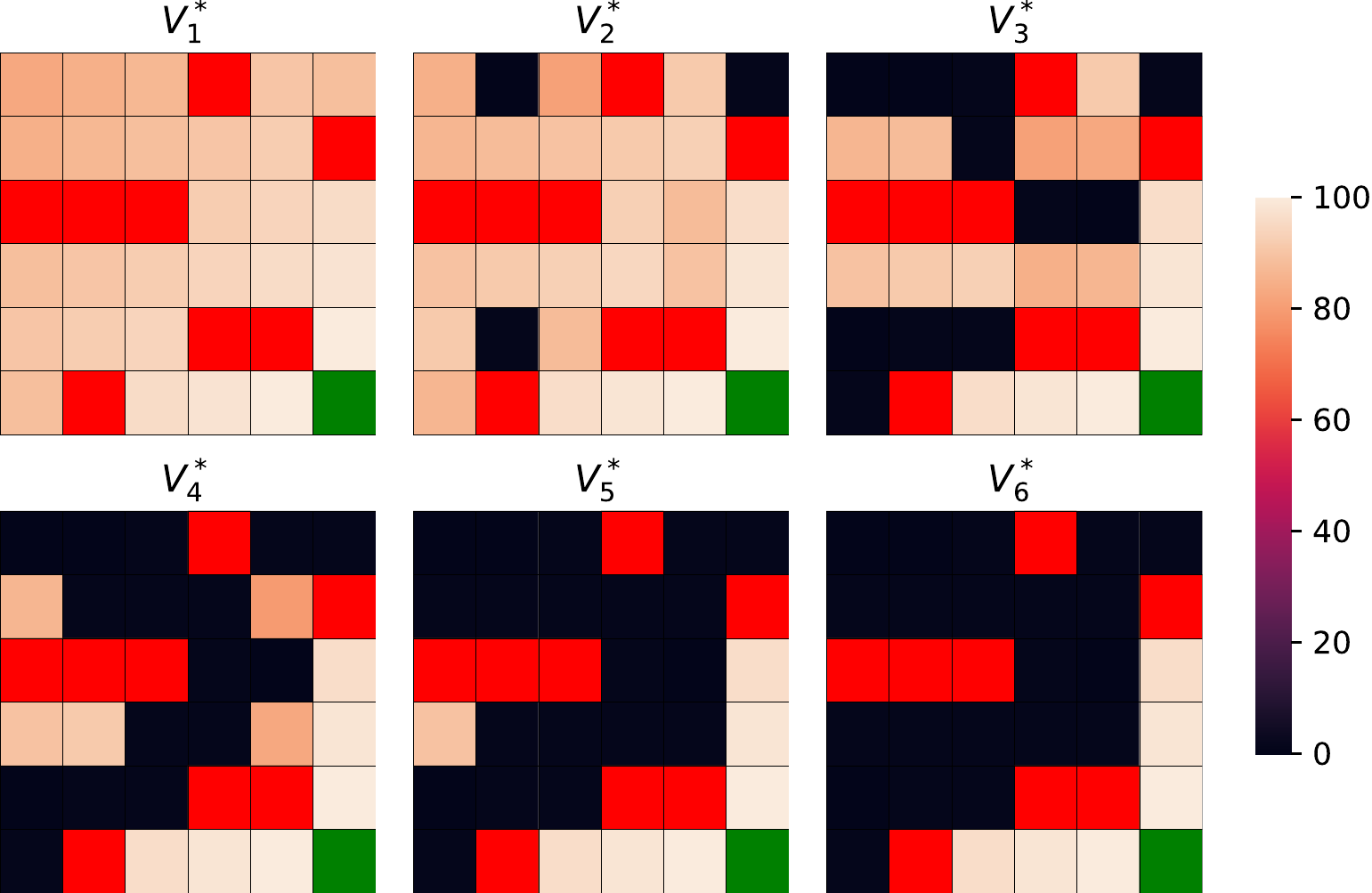}
\caption{$k$-value function representation for different persistence values $k$ ($K_{\max}=6$). Red cells denote holes, green cell the goal.}\label{fig:sync_val}
\end{figure}

In this appendix, we report some additional information related to the experiments with synchronous \algnameQL updates shown in Section \ref{sec:advantages}.

The environment considered is depicted in Figure \ref{fig:sync_env}: a 6x6 deterministic gridworld, with holes. The set of (primitive) actions is $\Aspace=\{left, down,right, up\}$. Transitions are deterministic. Going outside the borders and falling off a hole result in a punishment with a negative reward, respectively equal to -100 and -10 (hence, outer borders are not blocking the movement of the agent). The reward for reaching the goal instead is equal to +100. In all these cases the episode terminates; all other states result in a small negative reward (-1), to incentivize finding the shortest paths towards the goal. 

In order to exploit only the convergence properties of \algnameQL algorithm, without considering the exploration factor, we consider a synchronous learning framework: we assume to have access to the whole transition model, in such that, at each iteration $t$, we are able to collect an independent sample $s'\sim P(\cdot| s, a)$ for every state-action pair $(s, a) \in \SAs$. Since the environment is deterministic, this means that we have access to the whole transition matrix $P$. 
For each simulation, the value estimation for each state-action pair (or state-option pair, in the case of \algnameQL) is initialized sampling from a standard gaussian random variable. 
In each iteration of $Q$-learning, the algorithm performs a full update of the Q-function estimates.
In each iteration of \algnameQL, before performing the full update, the \textit{primitive} tuples are combined together, in order to collect a sample for each possible $(s,a,k)$ pair in $\SAKs$. 

In the plots on Figure \ref{fig:sync}, we represented the overall convergence to $Q^\star$ and the convergence of the  $Q_{\Kspace}$-value function restricted on the $k$-persistent actions $\mathcal{O}^{(k)}$, as specified in Section \ref{sec:advantages}.

The representations of the $k-$step value functions $V_k^{\star}$ are shown in Figure \ref{fig:sync_val}, where $V_k^{\star}(s)=\max_{a\in\Aspace}Q^{\star}(s,a,k)$. It is useful to remark that this value function do not coincide with the optimal value function in the $k-$persistent MDP $\mathcal{M}_k$, as $V_k^{\star}(s)$ represent the value function at the state $s$  restricted to persist $k$ times only the first action, and the following the optimal policy $\pi^{\star}$.

Parameters used for experiments:
\begin{itemize}\setlength\itemsep{.1em}
    \item Initial estimation: $Q(s,a,k)\sim\mathcal{N}(0,1)\ \forall(s,a,k)\in\SAKs$;
    \item Discount factor: $\gamma=0.99$;
    \item Learning rate: $\alpha=0.1$;
    \item Maximum number of iterations: 400 (plots truncated at 200).
\end{itemize}
\section{Details on Experimental Evaluation}\label{sec:apx exps}
In this appendix, we report more details about our experimental framework. In Appendix \ref{subsec:tabular_details} we provide more details about \algnameQL and the tabular setting; in Appendix \ref{subsec:mtcar_details} and \ref{subsec:atari_details} we focus on PerDQN and the deep RL experiments, respectively on Mountain Car and Atari games.

\paragraph{Infrastructure and computational times} The experiments have been run on a machine with two GPUs: a Nvidia Tesla K40c and a Nvidia Titan Xp with a  Intel(R) Core(TM) i5-4570 CPU @ 3.20GHz (4 cores, 4 thread, 6 MB cache ) and 32 GB RAM.

Computational times for Atari experiments: a single run of Freeway, with our infrastructure, needs around 4 hours for DQN, 5h 20 min for PerDQ and TempoRL.
The other Atari games require 10h for DQN, 12h for PerDQ and 13h 20min for TempoRL.

\subsection{Tabular Environments}\label{subsec:tabular_details}
\begin{figure}[t]
\centering
\subcaptionbox{Cliff}{\includegraphics[width=.2\columnwidth]{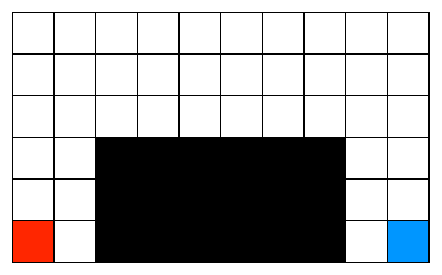}}
\subcaptionbox{Bridge}{\includegraphics[width=.2\columnwidth]{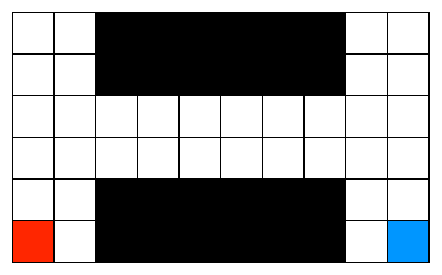}}

\subcaptionbox{Zigzag}{\includegraphics[width=.2\columnwidth]{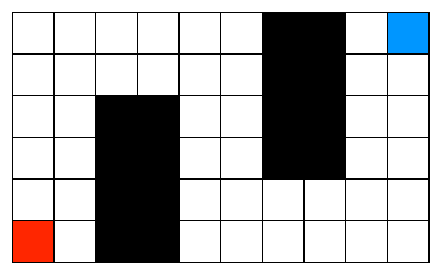}}
\caption{Tabular Gridworlds. Red cells denote the starting state and blue cells the goal state.}\label{fig:tabular}
\end{figure}

 The first environments tested are the deterministic 6x10 grids shown in Figure \ref{fig:tabular} and presented in \citealt{Biedenkapp2021TempoRL}. These environments are deterministic, and the outer borders block the agent from moving outside the grid (for example, an agent being at the top left cell will not move with an \textit{Up} action). Falling in the holes (black cells in Figure \ref{fig:tabular}) results in a $-1$ reward, while the goal is worth a positive reward, equal to $+1$. All other states have no reward. Reaching a Hole or the Goal terminates an episode.
 
Along with these three environments, we tested also the results on FrozenLake, available among OpenAi gym toolkit \cite{brockman2016open}. The transition process and the rewards are the same as in the previous case; the only differences are the bigger map (16x16) and the presence of random holes, such that each run is performed on a new map. For each new random map generated, the probability for each tile to be a hole (or frozen, according to the environment description) is equal to $0.85$.

\paragraph{Parameters:}
\begin{itemize}\setlength\itemsep{.1em}
    \item Initial estimation: $Q(s,a,k)\sim\mathcal{N}(0,1)\ \forall(s,a,k)\in\SAKs$;
    \item Discount factor: $\gamma=0.99$;
    \item Learning rate: $\alpha=0.01$;
    \item Maximum number of iterations: 6000 for FrozenLake, otherwise 600;
    \item Random policy probability: Exponentially decreasing: $\epsilon_t=0.99^t$.
\end{itemize}

\subsection{Mountain Car}\label{subsec:mtcar_details}

For Mountain Car experiments, the architecture chosen is a MLP: the first two hidden layers, consisting of 128 rectifier units, are shared among all persistences. The third hidden layer instead is diversified for each persistence value $k$, and each one is composed of 64 rectifier neurons and connected to three outputs, one for each action with its own persistence value.

The parameters adopted for the experiments are the following:
\begin{itemize} \setlength\itemsep{.1em}
    \item Discount factor: $\gamma=1$;
    \item Maximum number of iterations: $6\times 10^5$ (truncated to $5\times 10^5$ in the plots);
    \item Batch size: 32 for each persistent value;
    \item Replay buffer size: 50000 for each persistent value;
    \item Random policy probability: linearly decreasing, starting from $\epsilon_0=1$, to a final value $\epsilon_f=0.01$, reached at $15\%$ of the total number of iterations;
    \item Target update frequency: every 1000 steps;
    \item Train frequency: 1;
    \item Gradient clip: 10;
    \item Learning starts: 1000 (2000 only for Freeway);
    \item Loss function: Huber loss;
    \item Optimizer: Adam, with learning rate $\alpha=10^{-4}$, $\beta_1=0.9$, $\beta_2=0.999$;
    \item Prioritized replay buffers, of size $5\times10^4$ fro each persistence value, and default prioritization coefficients $\alpha_p=0.6$, $\beta_p=0.4$;
\end{itemize}

\subsection{Atari}\label{subsec:atari_details}
For Atari games, the architecture chosen is based on that presented in \cite{mnih2015human}, with three convolutional layers. the first hidden layer takes as input an $84\times 84\times 4$ image and convolves 32 filters of $8\times 8$ with stride 4, with the application of rectifier nonlinearities. The second has 32 input channels, 64 output channels, a kernel size of 4 and a stride of 2, again with ReLu activations, as well as the third convolutional layer, with kernel size of 3 and a stride of 1, and 64 output channels.
The convolutional structured is shaerd among all persistences, while the fully-connected hidden layer, consisting of 128 rectifier units, is differentiated for each persistence value $k$. Each one of these layers is fully-connected with the output layer, with a single output for each possible action.

The OpenAi Gym environments used are in the \textit{deterministic-v0} version (e.g. \textit{FreewayDeterministic-v0}), which do not make merging operations among the 4 input frames, but consideres only the last one.

The parameters adopted are the following:
\begin{itemize} \setlength\itemsep{.1em}
    \item Discount factor: $\gamma=0.99$;
    \item Maximum number of iterations: $2.5\times 10^6$;
    \item Batch size: 32 for each persistent value;
    \item Replay buffer size: 50000 for each persistent value;
    \item Random policy probability: linearly decreasing, starting from $\epsilon_0=1$, to a final value $\epsilon_f=0.01$, reached at $17\%$ of the total number of iterations;
    \item Target update frequency: every 500 steps;
    \item Train frequency: 1;
    \item Gradient clip: 10;
    \item Learning starts: 1000 (2000 only for Freeway);
    \item Loss function: Huber loss;
    \item Optimizer: Adam, with learning rate $\alpha=5\times10^{-4}$, $\beta_1=0.9$, $\beta_2=0.999$;
    \item Prioritized replay buffers, of size $5\times10^4$ from each persistence value, and default prioritization coefficients $\alpha_p=0.6$, $\beta_p=0.4$.
\end{itemize}

\clearpage
\section{Additional Results}
In this appendix, we report more experimental results. 
\subsection{Tabular environments}\label{sec:further_tab}
\begin{figure}
\centering
\includegraphics[width=\textwidth]{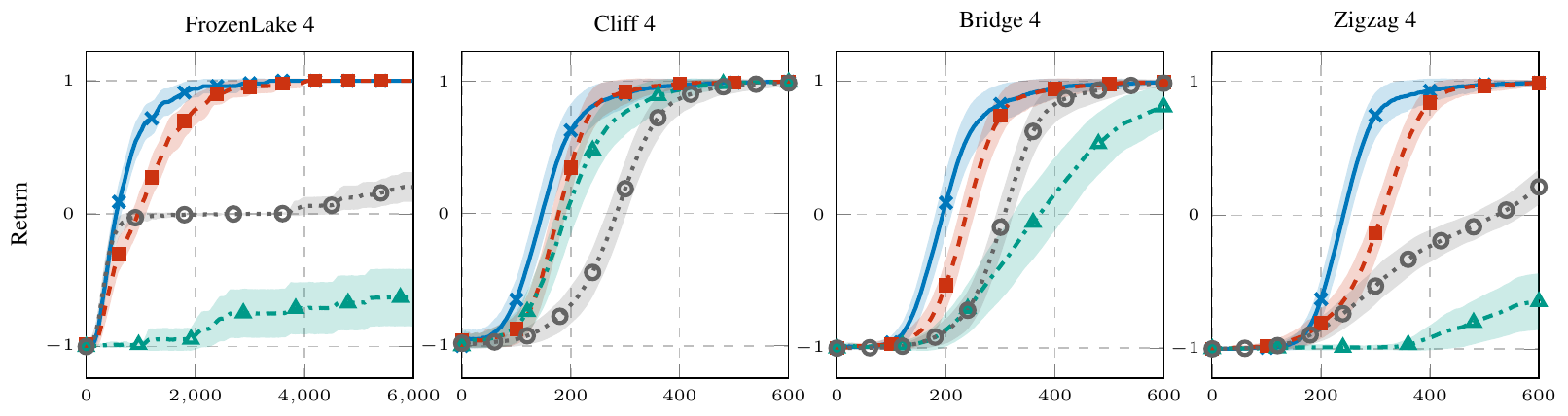}
\includegraphics[width=\textwidth]{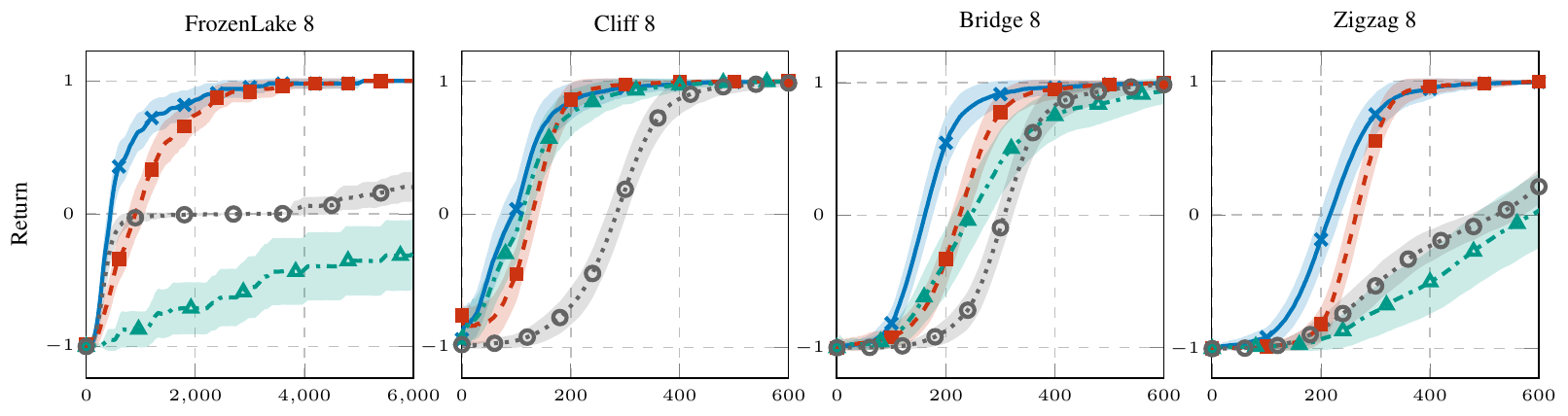}
\includegraphics[width=\textwidth]{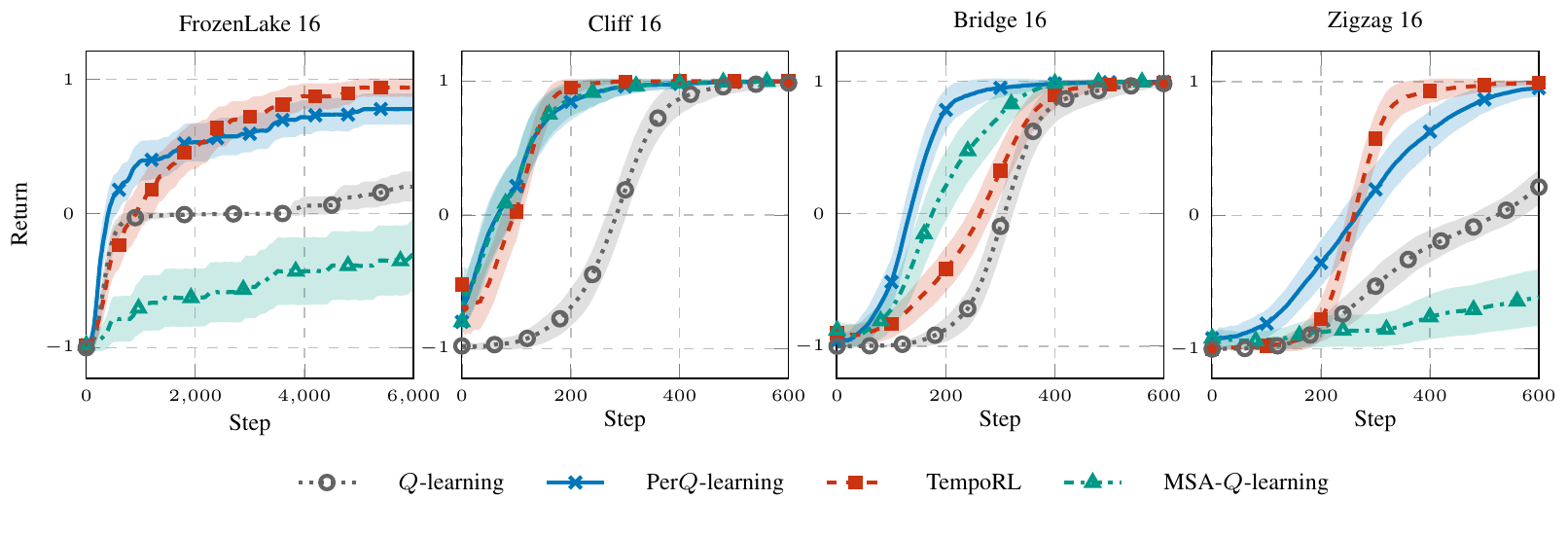}
\caption{Results of $Q$-learning, TempoRL, Per$Q$-learning and MSA-$Q$-learning   in  different tabular environments and maximum persistences. On each row, a different maximum persistence is selected for both algorithms. 50 runs (avg$\pm$ 95\% c.i.).}\label{fig:tabular_noboot}
\end{figure}

Here we consider the results related to the tabular environments. In Figure \ref{fig:tabular_noboot} we present the full comparison among \algnameQL-learning and TempoRL. As said in Section \ref{sec:exps}, we can observe a general faster convergence of the former. However, TempoRL is more robust with higher persistences, as in complex environments such as ZigZag and FrozenLake the performance does not degrade as \algnameQL learning.

Furthermore we analyzed the impact of the Bootstrap operator: as an ablation study, we performed the same \algnameQL learning evaluation without this feature. As a result, we can see that the returns are dramatically worse. Often, they perform even worse than classic Q learning: the reason of this behavior is due to the fact that, if a certain persistence value $k$ is not feasible for a state action pair $(s,a)$ (because of the geometry of the environment), its related value estimation $Q(s,a,k)$ will never be updated, and the algorithm may keep choosing it as the best one among the others. In TempoRL, albeit the absence of a Bootstrap, the update of the \textit{skip-value} function is instead updated by using the standard action-value function, improving the estimation.

\subsection{\algnameDQN additional results}\label{sec:further_dqn}
\begin{figure}
\centering
\includegraphics[width=.8\textwidth]{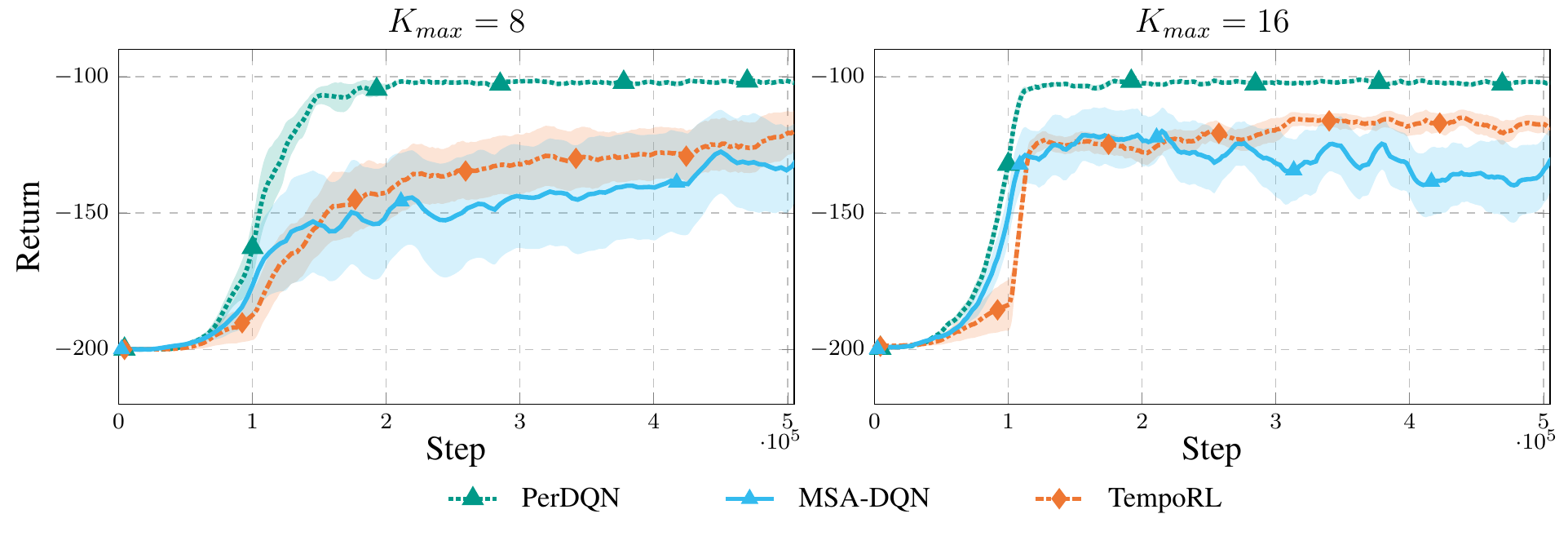}
\caption{\algnameDQN additional results on MountainCar. Return comparison of PerDQN with and without bootstrap (MSA-DQN) and TempoRL. 20 runs (avg $\pm$ 95\%c.i.)}\label{fig:mt_car_2}
\end{figure}
\begin{figure}
\centering
\includegraphics[width=.5\textwidth]{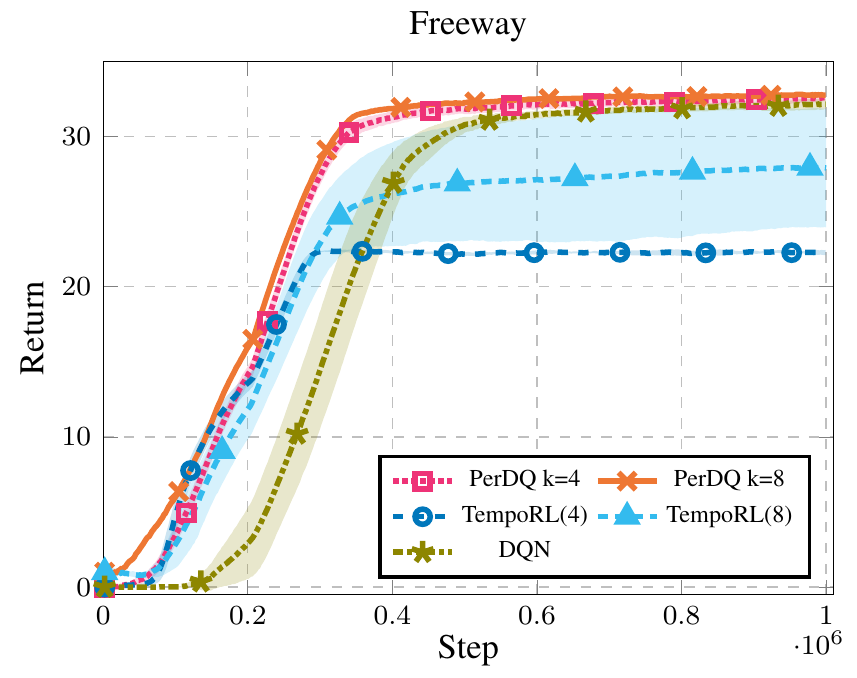}
\caption{\algnameDQN additional results on Freeway. Return comparison of PerDQN and TempoRL with maximum persistence $4$ and $8$. Parenthesis denote the maximum persistence chosen. 5 runs (avg $\pm$ 95\%c.i.)}\label{fig:freeway_extended}
\end{figure}

\begin{figure*}[t]
\centering
\includegraphics[width=\textwidth]{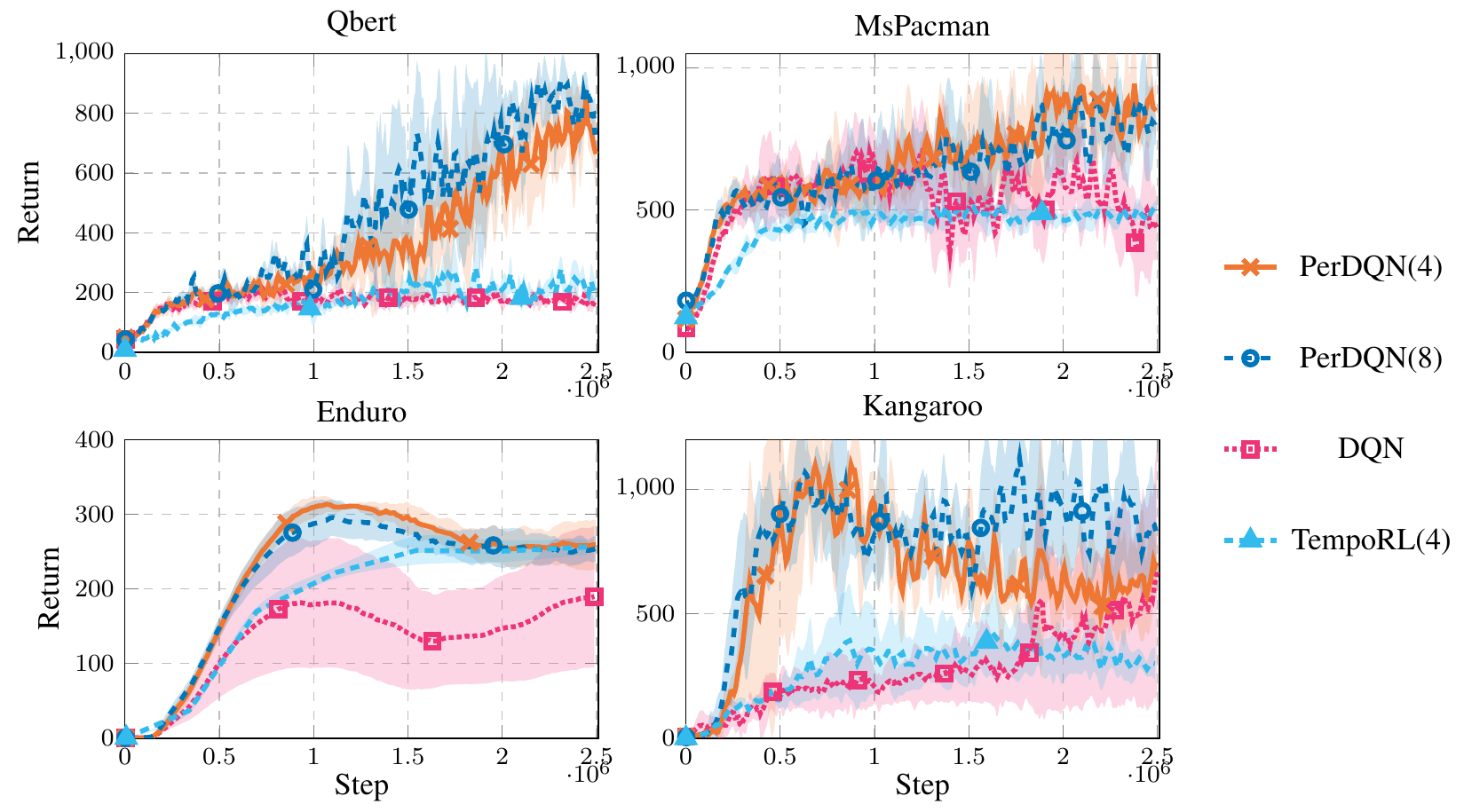}
\caption{Atari games results for DQN and PerDQN, with $K_{\max}=4$ and $8$. 5 runs (avg$\pm$ 95\% c.i.). }\label{fig:atariK}
\end{figure*}

\textbf{\algnameDQN with increased $K_{max}$ in Atari games}
Here we investigate the effects of an increased value of $K_{max}$ in Atari games: in Figure \ref{fig:freeway_extended} we show the results obtained on Freeway, by increasing the maximum persistence to 8 for both TempoRL and \algnameDQN. The same $K_{max}=8$ for \algnameDQN has been tested on the other Atari games, and shown in Figure \ref{fig:atariK} (we excluded Seaquest, as we have already shown that persistence in this environment is detrimental for learning.
As we can see, an increased maximum persistence $K_{max}=8$ does not provide improvements w.r.t. $K_{max}=4$ for the majority of the environments. A slight improvement can be seen for Freeway, while TempoRL, albeit providing a better learning curve than the one obtained with skip size equal to 4, is still unable to converge to the optimal policy.
\newline

\noindent \textbf{MountainCar, \algnameDQN without bootstrap}
Here we investigate the effects of the Bootstrap operator on \algnameDQN. In order to do so, we experimentally evaluated the algorithm without this feature on MountainCar and on the Atari games where persistence seemed to have beneficial effects (hence, with the exclusion of Seaquest environment). The results related to the Atari environments are shown in Figure \ref{fig:atari}, while MountainCar performances are shown in Figure \ref{fig:mt_car_2}: the version without bootstrap recalls the same contribution brought in \cite{schoknecht2003reinforcement} with MSA-Q-learning, hence it is here denoted as MSA-DQN. As we can see, bootstrap is an essential feature for \algnameDQN to converge rapidly to the optimal policy and to be robust. Indeed, without the bootstrap, the performances are worse, similar to TempoRL, and their variance is dramatically higher.

\subsection{\algnameQL: computational times}
\begin{figure*}[ht]
\centering
\includegraphics[width=\textwidth]{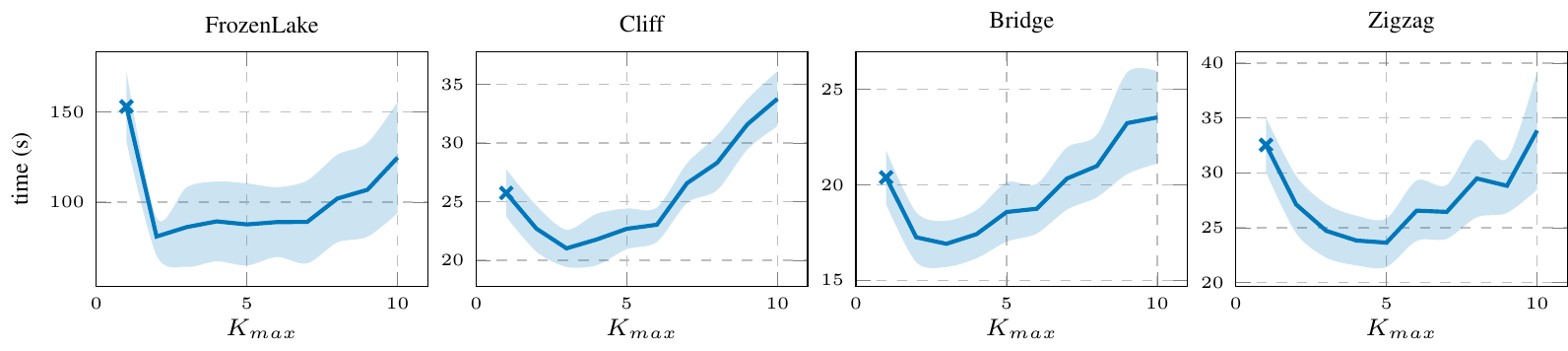}
\caption{Computational times required for a complete run of \algnameQL on tabular environments with different Maximum Persistences. 20 runs. avg $\pm$ 95\% c.i.}\label{fig:tabular_times}
\end{figure*}

We ran \algnameQL on the tabular  environments for different values of $K_{max}$ keeping track of the training time. We fixed the number of collected trajectories, but the trajectories can be of different lengths since the environments are goal-based. In Figure \ref{fig:tabular_times}, we observe that in all environments the minimum is attained by a value of $K_{max} > 1$. This means that the computational overhead due to using larger values of $K_{max}$ is compensated by a faster learning speed, leading to shorter trajectories overall. Note that $K_{max}=1$ corresponds to classic Q-learning. 
\clearpage 
\subsection{Experiments with the same number of updates and Comparison with $ez$-greedy exploration}

\begin{figure}[ht]
\centering
\includegraphics[width=\textwidth]{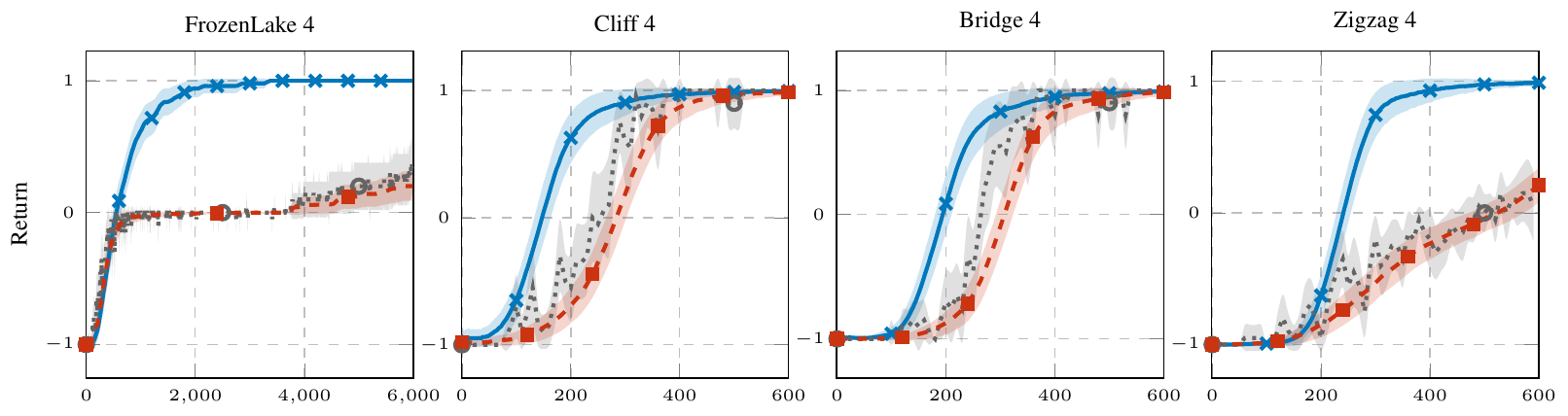}
\includegraphics[width=\textwidth]{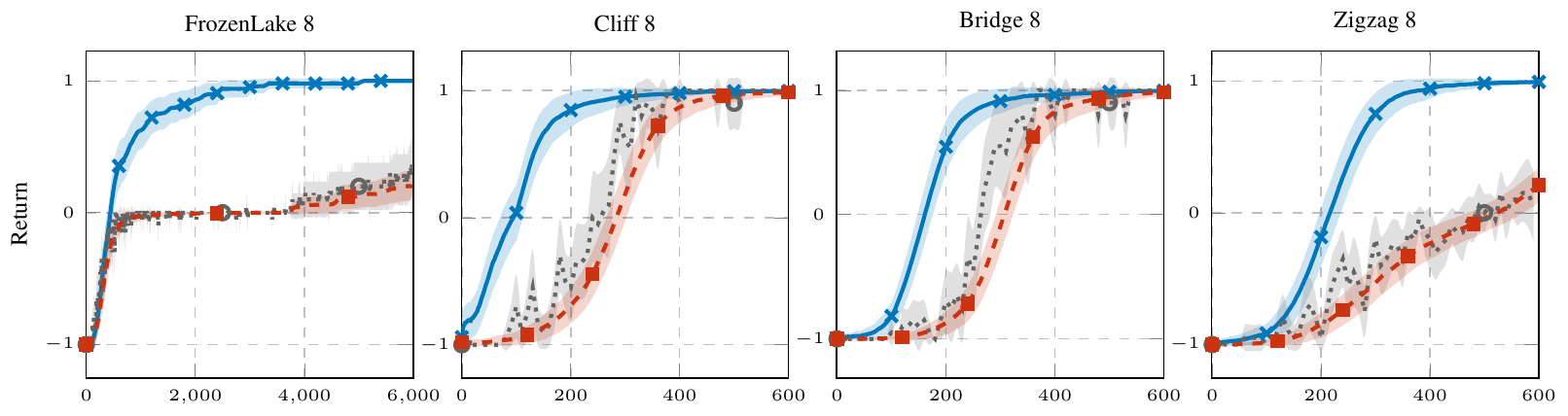}
\includegraphics[width=\textwidth]{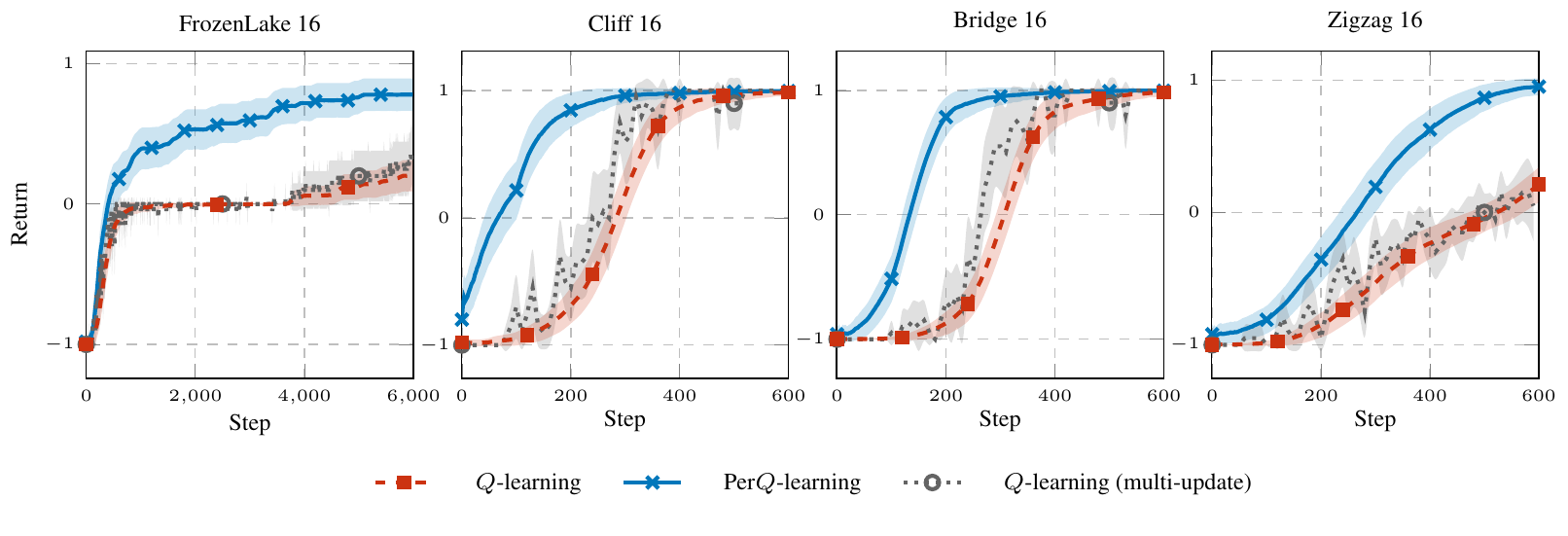}
\caption{Results of $Q$-learning, TempoRL, Per$Q$-learning and MSA-$Q$-learning   in  different tabular environments and maximum persistences. On each row, a different maximum persistence is selected for both algorithms. 50 runs (avg$\pm$ 95\% c.i.).}\label{fig:tabular_multi}
\end{figure}

\begin{figure}[ht]
\centering
\includegraphics[width=0.5\columnwidth]{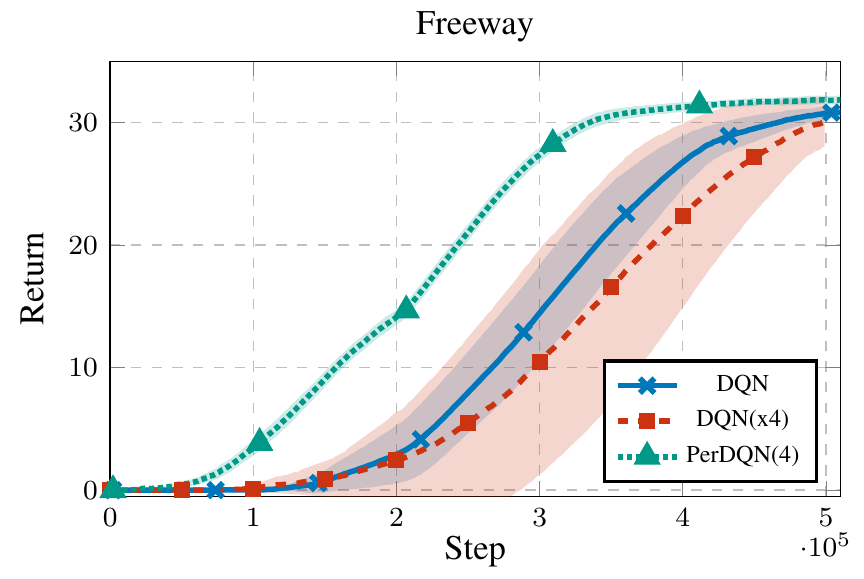}
\caption{Freeway results: \algnameDQN comparison with DQN with the same global amount of tuples sampled per update as \algnameDQN (DQN(x4) denotes DQN where the batch size is 4 times the one adopted for classic DQN, hence with the same global batch size as \algnameDQN with $K_{max}=4$).}
\label{fig:freeway_comp}
\end{figure}
\begin{figure}[ht]
\centering
\includegraphics[width=0.6\columnwidth]{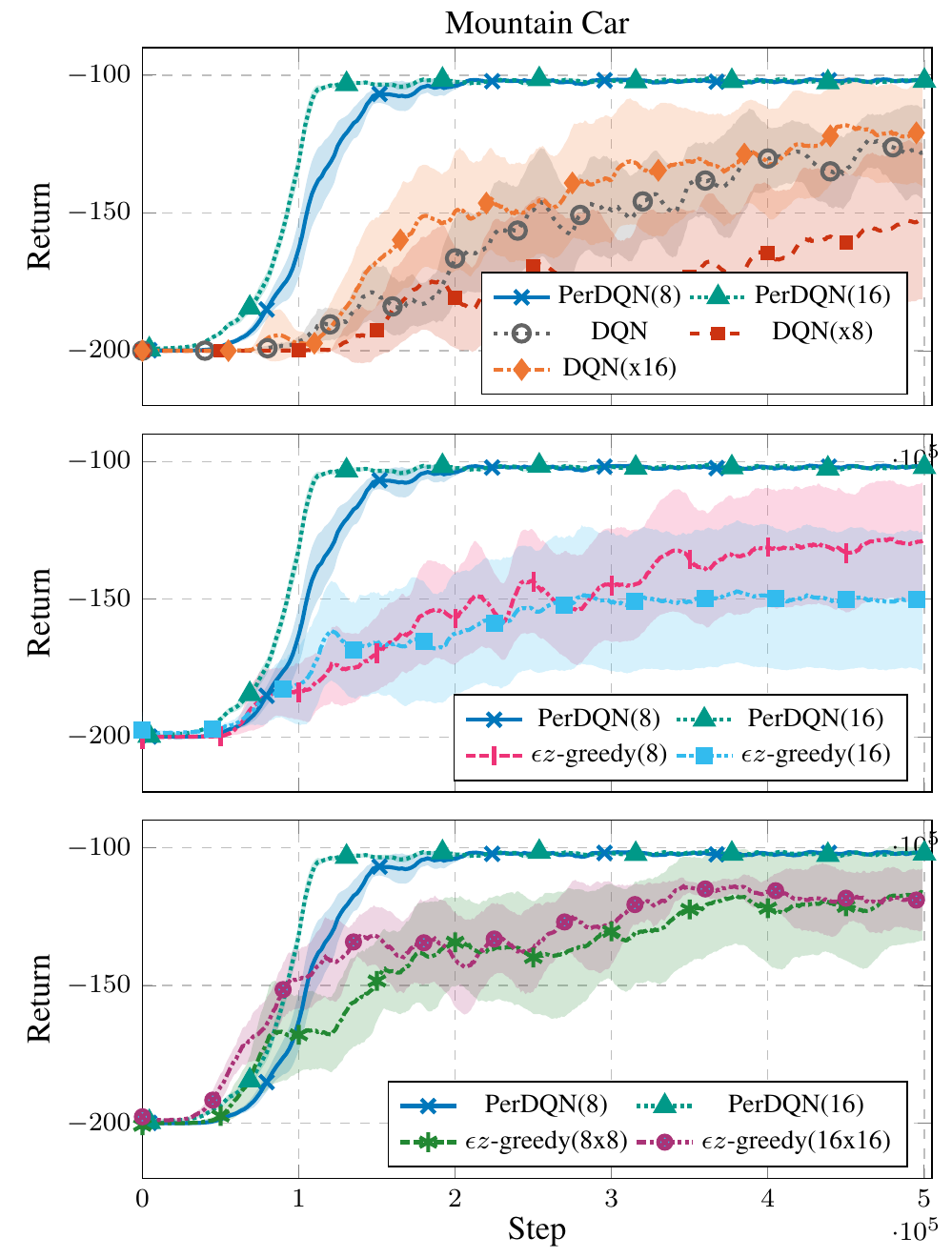}
\caption{MountainCar results for DQN, PerDQN and $\epsilon z$-greedy DQN (10 runs, avg$\pm$ 95\% c.i.). \\ Top Figure: \algnameDQN comparison with DQN with the same global amount of tuples sampled per update as \algnameDQN (e.g. DQN(x8) denotes DQN where the batch size is 8 times the one adopted for classic DQN, hence with the same global batch size as \algnameDQN with $K_{max}=8$). 
\\Middle figure:  \algnameDQN comparison with  $\epsilon z$-greedy DQN, where parenthesis denote the maximum persistence in the random sampling. 
\\Bottom figure:  \algnameDQN comparison with $\epsilon z$-greedy DQN, with the same global amount of tuples sampled per update as \algnameDQN (e.g. $\epsilon z$-greedy(8x8) denotes $\epsilon z$-greedy DQN with maximum persistence for exploration equal to 8, and the batch size is 8 times the one adopted for classic DQN, hence with the same global batch size as \algnameDQN with $K_{max}=8$)}\label{fig:mtcar_comp}
\end{figure}

In Algorithm \ref{alg:PerQ}, we can see that the number of total updates related is $\mathcal{O}(K_{max}^3)$. One might wonder if standard $Q$-learning can attain the same learning advantages with an increased number of updates: in Figure \ref{fig:tabular_multi}, we compare the results obtained for Per$Q$-learning with $Q$-learning, where the number of updates per each step has been increased for $K_{max}$ times (the row denotes the value of  $K_{max}$). Hence, the two algorithms have the same amount of updates. However, the increasing number of updates is not directly related to a learning improvement. If we consider $Q$-learning, increasing the number of updates (related to the same observed tuples) is equivalent to an increase in the learning rate, which is not necessarily related to a better learning performance. Indeed, the performances of $Q$-learning in its multiple-update version are not statistically better than the single-update version, only less robust.

Furthermore, the presence of multiple replay buffers allows to reduce the total number of updates linear in $K_{max}$ in \algnameDQN, which is the most suited for real-world scenarios with large state-spaces.   
In Figure \ref{fig:mtcar_comp}, we compare \algnameDQN with other baselines on the MountainCar environment: In the top plot, we show a comparison between our proposed approach and DQN with the same number of updates as PerDQN (denoted as DQN(x$K$)). In \algnameDQN, the batch size for each replay buffer is 32: in the new DQN runs, the total batch size for an update has been increased to 32*$K_{max}$.

In the middle plot, we compare \algnameDQN with a vanilla DQN employing $\epsilon z$-greedy exploration \cite{dabney2020temporally}. In the implementation of $\epsilon z$-greedy DQN, the exploration is performed in a similar manner as in \algnameDQN, as the persistence is sampled from a discrete uniform distribution in $1,\dots,K_{max}$, where $K_{max}$ has been set to 8 and 16.

In the bottom plot, $\epsilon z$-greedy DQN is run with the same global batch size per update as in \algnameDQN.

In Figure \ref{fig:freeway_comp}, we compare \algnameDQN ($K_{max}$=4) with DQN on Freeway, with an increased batch size of a factor 4, in such a way that the total number of samples per update is the same.

 In a similar fashion as with standard $Q$-learning, in Figure \ref{fig:mtcar_comp} the DQN curve related to $8x$ the sample size is worse (on average) than the standard version, while the $16x$ experiments see a slight improvement (with no statistical evidence). In any case, \algnameDQN outperforms all the compared methods. The same holds for Freeway (Figure \ref{fig:freeway_comp}), where augmenting the DQN batch size by a factor 4 does not provide improvements in the performances.

In the middle and bottom plots of Figure \ref{fig:mtcar_comp}, we can see that the exploration with persistence alone (through $\epsilon z$-greedy exploration) is not enough to provide the same improvement in the learning capabilities, differently from \algnameDQN. Furthermore, increasing the number of updates can slightly help learning, but the resulting learning curves are still largely dominated by \algnameDQN.

\end{document}